\documentclass[letterpaper]{article} 
\usepackage{arxiv}
\usepackage{times}  
\usepackage{helvet}  
\usepackage{courier}  
\usepackage[hyphens]{url}  
\usepackage{graphicx} 
\usepackage{natbib}  
\usepackage{caption} 
\usepackage{algorithm}
\usepackage{algorithmic}
\usepackage{amsmath}
\usepackage{amsthm}
\usepackage{amssymb}
\usepackage{amsfonts}
\usepackage{url}
\usepackage{multirow}
\usepackage{booktabs}
\usepackage{xr}

\newtheorem{theorem}{Theorem}
\newtheorem{lemma}{Lemma}
\newtheorem{assumption}{Assumption}
\newtheorem{definition}{Definition}

\newtheorem{corollary}{Corollary}
\newtheorem{proposition}{Proposition}
\usepackage{newfloat}
\usepackage{listings}
\DeclareCaptionStyle{ruled}{labelfont=normalfont,labelsep=colon,strut=off} 
\floatstyle{ruled}
\newfloat{listing}{tb}{lst}{}
\floatname{listing}{Listing}
\title{Unpacking the Implicit Norm Dynamics of \\ Sharpness-Aware Minimization in Tensorized Models}
\author{
Tianxiao Cao, Kyohei Atarashi, Hisashi Kashima\\
Graduate School of Informatics, Kyoto University, Kyoto, Japan\\
\texttt{cao.tianxiao.65w@st.kyoto-u.ac.jp, \{atarashi,kashima\}@i.kyoto-u.ac.jp}
}

\date{} 

\begin{document}

\maketitle

\begin{abstract}
Sharpness-Aware Minimization (SAM) has been proven to be an effective optimization technique for improving generalization in overparameterized models. While prior works have explored the implicit regularization of SAM in simple two-core scale-invariant settings, its behavior in more general tensorized or scale-invariant models remains underexplored. In this work, we leverage scale-invariance to analyze the norm dynamics of SAM in general tensorized models. We introduce the notion of \emph{Norm Deviation} as a global measure of core norm imbalance, and derive its evolution under SAM using gradient flow analysis. We show that SAM's implicit control of Norm Deviation is governed by the covariance between core norms and their gradient magnitudes. Motivated by these findings, we propose a simple yet effective method, \emph{Deviation-Aware Scaling (DAS)}, which explicitly mimics this regularization behavior by scaling core norms in a data-adaptive manner. Our experiments across tensor completion, noisy training, model compression, and parameter-efficient fine-tuning confirm that DAS achieves competitive or improved performance over SAM, while offering reduced computational overhead.
\end{abstract}

\section{Introduction}
The remarkable success of overparameterized deep neural networks has highlighted a central challenge in modern machine learning: generalization. These models possess enough capacity to memorize the entire training dataset, yet they often perform exceptionally well on unseen data~\cite{zhang2021understanding}. Understanding the mechanisms that prevent overfitting and promote generalization is a fundamental pursuit~\cite{srivastava2014dropout,belkin2019reconciling,ishida2020we,bisla2022low}.

A leading explanation involves the geometry of the loss landscape, where flatter minima are empirically associated with better generalization. Sharpness-Aware Minimization (SAM) \cite{foret2020sharpness} is a widely adopted optimization method that materializes this idea by minimizing the worst-case loss within a small perturbation neighborhood. SAM has demonstrated robust performance across a variety of tasks~\cite{kwon2021asam,bahri2022sharpness,ilbert2024samformer}, and its theoretical properties have been explored in depth \cite{wen2022sharpness,baek2024why,andriushchenko2022towards}, though mostly for standard, dense architectures.

Concurrently, there has been a long-standing interest in parameter-efficient modeling approaches, where structure is imposed on weights to reduce memory and computation for efficient and eco-friendly machine learning~\cite{memmel2024position}. These include tensor-decomposed layer~\cite{novikov2015tensorizing,hrinchuk2020tensorized} or low-rank adaptation (LoRA) modules~\cite{jie2023fact,yang2024loretta,yaras2024compressible,si2025maintaining,veeramacheneni2025canonical}, which offer significant compression and computational savings by modeling parameters as combinations of small tensor cores. Such structured parameterizations introduce scale-invariance and inter-dependencies that can interact non-trivially with optimization dynamics.

While prior work has shown that SAM implicitly promotes norm balancing in simple matrix factorization setups~\cite{li2024implicit}, its behavior in general multi-core tensorized models, where factor norms can differ drastically and influence training stability, is far less understood. This gap motivates our central research question: 

\textit{\textbf{How does the implicit regularization of SAM manifest in general, multi-core, scale-invariant tensorized models?}}

To answer the question, we investigate the implicit norm dynamics of both SGD and SAM applied on the general scale-invariant problem. We propose a global measure Norm Deviation $Q$ (Definition~\ref{def:global-norm-deviation}) to capture the norm imbalance. We present theoretical (Theorem~\ref{thm:sam-norm-deviation}) and empirical analysis (Fig.~\ref{fig:tucker}) and show that $Q$ is governed by the covariance between core norms and gradient norms, and that this effect is amplified by data noise. Inspired by the theoretical findings, we proposed a computationally efficient alternative of SAM, Deviation-Aware Scaling, by mimicking the norm dynamics through core scaling without computing the adversarial perturbation. We experimented on comprehensive tasks related to tensor-based parameterization. To summarize, our contributions are as follows:
\begin{itemize}
\item Norm Deviation $Q$ is proposed a global proxy for imbalance among tensor cores, and analyze its behavior under different optimizers. We prove conditions for local norm shrinkage and study its data-dependent dynamics.

\item We propose Deviation-Aware Scaling (DAS), a novel optimizer that avoids the adversarial perturbation step of SAM by distilling its norm effects into explicit scaling.

\item We empirically validate our findings across various domains, including tensorized neural networks, few-shot language model finetuning with low-rank adapters, and noisy-label learning. Our results show that SAM consistently improves performance in these settings, and that DAS inherits many of these benefits while reducing computational overhead.

\end{itemize}

\section{Related Work}

\paragraph{SAM and mechanism of SAM.} With the success of SAM, there emerged a line of extending works, including addressing the computational efficiency~\cite{du2021efficient,du2022sharpness,mueller2023normalization,ji2024single,xie2024sampa,deng2025asymptotic} and improving the performance~\cite{kwon2021asam,zhuang2022surrogate,kim2022fisher,liu2022random}. To explain the success of SAM, great efforts are made to understand SAM, mainly through simplified modeling~\cite{wen2022sharpness,compagnoni2023sde,bartlett2023dynamics}, empirical observations~\cite{andriushchenko2023modern}, and proxy measures~\cite{li2024implicit}.

\paragraph{Tensor decomposition in model compression.} Multi-core structured models or tensor decompositions have long been a promising direction for model compression. Existing works utilize decompositions including Tensor-Train~\cite{novikov2015tensorizing,yin2021towards,loeschcke2024coarse}, Tensor-Ring~\cite{wang2018wide,li2022heuristic,cao2024learning}, Tucker~\cite{phan2020stable}, and even arbitrary tensor networks~\cite{hayashi2019exploring}, achieving a practical reduction in storage through structured parameterization.

\paragraph{Tensor-based LoRA.} As large pre-trained models have become the standard, fine-tuning them for downstream tasks has become computationally prohibitive. LoRA~\cite{hu2022lora} injects low-rank trainable matrices into the layers of a Transformer, drastically reducing the number of trainable parameters. Recent advanced LoRA variants explore high-order tensor decomposition, such as CP~\cite{veeramacheneni2025canonical}, Tensor-Train~\cite{jie2023fact}, Tucker~\cite{jie2023fact,si2025maintaining}, and deep matrix factorization~\cite{yaras2024compressible}.

\section{Problem Statement}
\paragraph{Notations.} 
A scalar and a tensor (including matrices as second-order tensors) are denoted by 
$x$ and $\mathcal X$, respectively, unless otherwise specified.
For two tensors $ \mathcal A, \mathcal B \in \mathbb{R}^{n_1 \times \cdots \times n_d} $, we define the \emph{Frobenius inner product} as
$\left\langle \mathcal A, \mathcal B \right\rangle_F := \sum_{i_1=1}^{n_1} \cdots \sum_{i_d=1}^{n_d} \mathcal A_{i_1,\dots,i_d} \cdot \mathcal B_{i_1,\dots,i_d}$.
We denote by $\|\mathcal G\|^2_F = \left\langle \mathcal G, \mathcal G \right\rangle_F$, the squared Frobenius norm of a tensor $\mathcal G$, 
\textit{i.e.}, the sum of squares of all entries of $\mathcal G$. For a positive integer $n$, we denote $[n] := \{1, 2, \ldots, n\}$.

We consider a class of general scale-invariant models.
The parameters consist of a set of core tensors $\{\mathcal G_k\}_{k=1}^K$. 
These cores are composed via a multilinear reconstruction function 
$\Phi(\mathcal G_1, \ldots, \mathcal G_K)$ that produces the full tensor
$\mathcal T \in \mathbb{R}^{n_1 \times \cdots \times n_d}$. The problem of interest is to find the solution to the following optimization problem:
\begin{equation}
\label{eq:problem}
\min_{\mathcal G_1, \ldots, \mathcal G_K} f(\mathcal T) 
= f(\Phi(\mathcal G_1, \ldots, \mathcal G_K)),
\end{equation}
where $f(\cdot)$ is a scalar function. The key structural property is that the reconstruction function \( \Phi \) is \emph{multilinear} in each of cores.
That is, for any core index \( k \in \{1, \ldots, K\} \), and any scalar \( \lambda \in \mathbb{R} \), we have:
\[
\Phi(\mathcal{G}_1, \ldots, \lambda \mathcal{G}_k, \ldots, \mathcal{G}_K) = \lambda \cdot \Phi(\mathcal{G}_1, \ldots, \mathcal{G}_K).
\]
Specifically, scale-invariance refers to that $\{c_k \mathcal G_k\}_{k=1}^K$ have the same reconstructed tensor and objective value $\forall c_k\in \mathbb R,\quad \prod _{k=1}^K c_k = 1$. Such reconstruction functions arise naturally in a broad class of tensor decomposition models and tensor networks, where the output tensor \( \mathcal{T} \) is constructed via a sequence of tensor contractions\footnote{This formulation encompasses common structures such as Tucker, CP, Tensor-Train, and Tensor-Ring, as well as tensorized neural network layers where the weight tensor is parameterized via structured factorization.} over shared modes. 

\subsection{Sharpness-Aware minimization}
Sharpness-Aware minimization (SAM) \cite{foret2020sharpness} has emerged as a powerful technique 
for improving the generalization of various machine learning models by seeking flat solutions. 
SAM optimizes the worst-case 
loss over a neighborhood of the parameters, with which Problem~\eqref{eq:problem} can be reformulated as:
\begin{equation}
\label{eq:sam}
\min_{\mathcal{G}_1, \ldots, \mathcal{G}_K} \ \max_{\| \Delta \| \leq \rho}
f\left(\Phi(\mathcal{G}_1 + \Delta_1, \ldots, \mathcal{G}_K + \Delta_K)\right),
\end{equation}
where $\|\Delta\| := (\sum_{k=1}^K \|\Delta_k\|_F^2)^{1/2}$ and $\rho$ is the radius of perturbation. A practical implementation of SAM,
directly adapted from \cite{foret2020sharpness} is given in Algorithm \ref{alg:sam-cores}.
\begin{algorithm}[tb]
\caption{SAM for Problem \eqref{eq:sam} on tensorized models}
\label{alg:sam-cores}
\begin{algorithmic}[1]
\STATE \textbf{Initialize:} tensor cores $\{\mathcal{G}_k^{(0)}\}$, step size$\{\eta^{(t)}\}$, radius of perturbation $\rho$, and number of iterations $T$.
\FOR{$t = 0, \ldots, T-1$}
    \STATE Compute $g_k = \nabla_{\mathcal{G}_k} f(\Phi(\mathcal{G}_1^{(t)}, \ldots, \mathcal{G}_K^{(t)})), \forall k \in [K]$
    \FOR{$k = 1, \ldots, K$}
    \STATE $\tilde{\mathcal{G}}_k^{(t)} = \mathcal{G}_k^{(t)} + \rho \frac{g_k}{\sqrt{\sum_{j=1}^K \|g_j\|_F^2}}$
    \STATE $\tilde{g}_k = \nabla_{\mathcal{G}_k} f(\Phi(\tilde{\mathcal{G}}_1^{(t)}, \ldots, \tilde{\mathcal{G}}_K^{(t)}))$
    \STATE Update $\mathcal{G}_k^{(t+1)} = \mathcal{G}_k^{(t)} - \eta^{(t)} \tilde{g}_k$  via Adam or SGD
    \ENDFOR
\ENDFOR
\end{algorithmic}
\end{algorithm}
\paragraph{Mechanism understanding with norm.}
Analyzing the implicit bias of optimization algorithms on structured models by examining the norm of each structured component is becoming a powerful tool, as in the matrix factorization case \cite{liu2021noisy, li2024implicit}.
In the context of SAM, \textit{balancedness}~\cite{li2024implicit}, \textit{i.e.}, the difference in the squared Frobenius norm of two factors in the matrix factorization case, is introduced as a global metric to characterize the implicit regularization of SAM.
Motivated by this, we aim to examine the norm behaviors of cores in different optimizers and provide theoretical insights.

\section{Analysis of Norm Dynamics for General Scale-Invariant Problems}
Due to space constraints, we provide the proofs for all theoretical results in the Appendix. Following previous works on implicit regularization~\cite{arora2019implicit,gunasekar2017implicit,li2024implicit}, 
we consider gradient flow with infinitesimal stepsize $\eta \to 0$ for Problem~\eqref{eq:problem}:
\begin{equation}\label{eq:sgd-update-step}
\begin{aligned}
    &g_k^{(t)} = \nabla_{\mathcal G_k} f(\Phi(\mathcal G_1^{(t)}, \ldots, \mathcal G_K^{(t)})), \\
    &\mathcal G_k^{(t+1)} = \mathcal G_k^{(t)} - \eta g_k^{(t)},\quad \forall k\in [K].
\end{aligned}
\end{equation}

In the infinitesimal stepsize limit \( \eta \to 0 \), 
time $t$ is treated as continuous.
We show that the dynamics of the squared norm among all cores are equivalent.
\begin{theorem}[Norm Dynamics of SGD]
\label{thm:equivalent-norm}
Applying the update step Equation \eqref{eq:sgd-update-step} with infinitesimal stepsize $\eta \to 0$ to the Problem \eqref{eq:problem}, the time derivatives of the squared Frobenius norms of all cores are equal:
$$
\frac{d}{dt}\|\mathcal G_1^{(t)}\|^2_F = \cdots = \frac{d}{dt}\|\mathcal G_K^{(t)}\|^2_F.
$$
\end{theorem}

Next, we tackle the dynamics of the squared Frobenius norm in SAM.
The update step of SAM in Algorithm \ref{alg:sam-cores} can be rewritten as:
\begin{equation}
\label{eq:sam-update-step}
\begin{aligned}
    &g_k^{(t)} = \nabla_{\mathcal G_k} f(\Phi(\mathcal G_1^{(t)}, \ldots, \mathcal G_K^{(t)})), \\
    &\tilde{\mathcal G}_k^{(t)} = \mathcal G_k^{(t)} + \rho u^{(t)} g_k^{(t)},\\
    &\tilde{g}_k^{(t)} = \nabla_{\mathcal G_k} f(\Phi(\tilde{\mathcal G}_1^{(t)}, \ldots, \tilde{\mathcal G}_K^{(t)})),\\
    &\mathcal G_k^{(t+1)} = \mathcal G_k^{(t)} - \eta \tilde{g}_k^{(t)},\quad \forall k\in [K],
\end{aligned}
\end{equation}
where $u^{(t)} = (\sum_{j=1}^K \|g_j^{(t)}\|_F^2)^{-1/2}$ is the normalization factor, and $\rho>0$ is the radius of the perturbation.

We use a standard assumption on the Lipschitz smoothness of $f(\cdot)$ following
previous works on analyzing SAM~\cite{andriushchenko2022towards,wen2022sharpness,mueller2023normalization,li2024implicit} and optimization~\cite{ge2015escaping,wang2024generalized}:
\begin{assumption}[Smoothness]\label{assumption:smoothness}
There exists $L > 0$ such that for any $\mathcal X, \mathcal Y\in \mathbb{R}^{n_1 \times \cdots \times n_d}$, it holds that
$$\|\nabla f(\mathcal{X}) - \nabla f(\mathcal{Y})\|_F \leq L \|\mathcal{X} - \mathcal{Y}\|_F.$$
\end{assumption}

\citeauthor{li2024implicit} analyzed the dynamics of \textit{difference} between the squared Frobenius norms of two cores in the matrix factorization case. Under the multi-core setting in our paper, it is natural to derive the dynamics for the difference between the squared norms of two cores, as follows:
\begin{theorem}[Pairwise Norm Dynamics under SAM]
\label{thm:sam-dynamics}
Applying the update steps \eqref{eq:sam-update-step} with infinitesimal stepsize $\eta \to 0$, the gradient flow
of SAM satisfies that $\forall i,j \in [K]$ with $i\neq j$:
\begin{align*}
\frac{d}{dt}\left(\|\mathcal G_i^{(t)}\|_F^2 - \|\mathcal G_j^{(t)}\|_F^2\right) 
&= 2\rho u^{(t)} \left(\|g_i^{(t)}\|_F^2 - \|g_j^{(t)}\|_F^2\right) \\
&\quad + O(\rho^2 L).
\end{align*}
\end{theorem}

Theorems~\ref{thm:equivalent-norm} and~\ref{thm:sam-dynamics} can 
be viewed as non-trivial multi-core generalizations of the results 
in~\cite{li2024implicit}, which focused on matrix factorization. 
Specifically, Theorem~\ref{thm:equivalent-norm} shows that under 
gradient flow, standard SGD preserves the difference in squared 
Frobenius norms between any two cores, \textit{i.e.}, 
$\|\mathcal G_i\|_F^2 - \|\mathcal G_j\|_F^2$ remains constant. 
In contrast, Theorem~\ref{thm:sam-dynamics} reveals that SAM 
introduces a dynamic regulation mechanism: 
the rate of change in $\|\mathcal G_i\|_F^2 - \|\mathcal G_j\|_F^2$ 
is proportional to the difference in squared gradient norms, 
$\|g_i\|_F^2 - \|g_j\|_F^2$.
These results provide important insights into how individual core norms evolve relative to each other. In the multi-core setting under SAM, however, pairwise norm differences are local and not sufficient to characterize the overall norm dynamics of the whole system. To fully understand the norm behavior of all cores under SGD and SAM, we need a new measure to analyze the \textit{global} norm dynamics to capture the collective norm dynamics of the model.
\paragraph{A global measure.} While directly using pairwise norm differences is appealing, summing or linearly combining them leads to cancellation and fails to provide a meaningful global measure. For instance, with three cores where their squared norms are $\|\mathcal{G}_1\|_F^2=2$, $\|\mathcal{G}_2\|_F^2=10$, and $\|\mathcal{G}_3\|_F^2=18$, a circular sum of differences $(\|\mathcal{G}_1\|_F^2 - \|\mathcal{G}_2\|_F^2) + (\|\mathcal{G}_2\|_F^2 - \|\mathcal{G}_3\|_F^2) + (\|\mathcal{G}_3\|_F^2 - \|\mathcal{G}_1\|_F^2)$ yields $(-8) + (-8) + (16) = 0$, masking the significant underlying imbalance. 
To address this, we introduce a global measure that captures the collective norm dynamics of all cores.
\begin{definition}[Global Squared Norm Deviation]
\label{def:global-norm-deviation}
For cores $\{\mathcal G_k\}_{k=1}^K$, the global squared norm deviation is defined as:
\begin{equation}
\label{eq:global-norm-deviation}
Q := \sum_{k=1}^K \left(\|\mathcal G_k\|_F^2 - \frac{1}{K}\sum_{i=1}^K \|\mathcal G_i\|^2_F\right)^2.
\end{equation}
For brevity, we will refer to this quantity as the \emph{Norm Deviation} when the context is clear.
\end{definition}

The Norm Deviation $Q$ captures the spread of the squared Frobenius norms of all cores. A larger $Q$ indicates greater imbalance among the cores. All cores have the same Frobenius norm if and only if $Q=0$.
Also, it can be verified that the Norm Deviation $Q$ has an alternative expression:
\begin{equation}\label{eq:sqsum-pairwise}
Q = \frac{1}{2K}\sum_{i,j=1}^K \left(\|\mathcal G_i\|_F^2 - \|\mathcal G_j\|_F^2\right)^2.
\end{equation}
This shows that $Q$ aggregates all the pairwise imbalances in the squared Frobenius norms of the cores.
Next, we derive the dynamics of the Norm Deviation $Q$ under SGD and SAM. 
The Norm Deviation $Q$ of SGD is conserved, from a direct application of 
Theorem~\ref{thm:equivalent-norm}:
\begin{corollary}[Norm Deviation Dynamics under SGD]
\label{cor:sgd-norm-deviation}
Applying the update steps \eqref{eq:sgd-update-step} with infinitesimal stepsize 
$\eta \to 0$ to the problem \eqref{eq:problem},
$$
\frac{dQ}{dt} = 0.
$$
\end{corollary}
The dynamics of $Q$ under SAM is as follows:
\begin{theorem}[Norm Deviation Dynamics under SAM]
\label{thm:sam-norm-deviation}
Applying the update steps \eqref{eq:sam-update-step} with infinitesimal stepsize $\eta \to 0$, the gradient flow of SAM satisfies:
\[
\frac{dQ}{dt} = 4\rho u^{(t)} K \cdot \mathrm{Cov}\left(\|\mathcal G_k^{(t)}\|_F^2, \|g_k^{(t)}\|_F^2\right) + O(\rho^2 L),
\]
where \( \mathrm{Cov}(x_k, y_k) := \frac{1}{K} \sum_{k=1}^K (x_k - \bar{x})(y_k - \bar{y}) \)
denotes the empirical covariance, with \( \bar{x}, \bar{y} \) the means over \( k \in [K] \).
\end{theorem}
Theorem~\ref{thm:sam-norm-deviation} shows that the Norm Deviation $Q$ under SAM evolves according to the covariance between the squared Frobenius norms of the cores and their corresponding gradient norms. Specifically, if $\|g_k\|_F^2$ is large when $\|\mathcal G_k\|_F^2$ is small (negative covariance), SAM grows small cores faster and helps all cores to encourage equalization in their norms. 
\begin{figure}[t]
\centering
\includegraphics[width=\columnwidth]{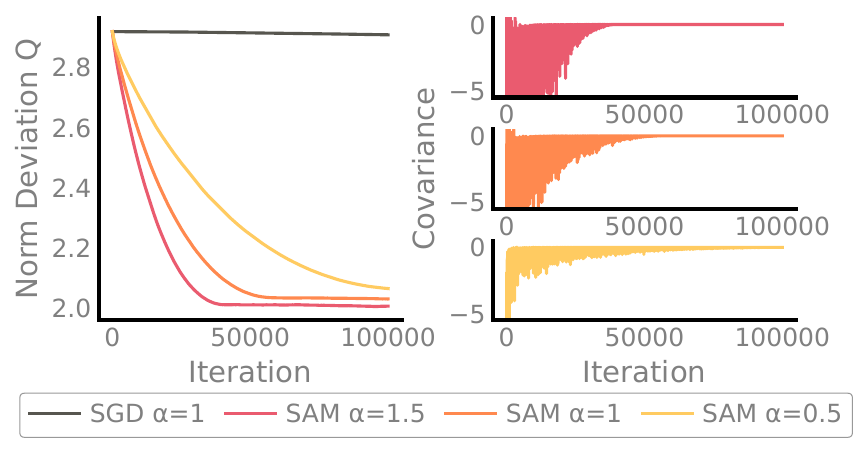}
\caption{
    Implicit regularization of SAM on norm deviation (Definition~\ref{def:global-norm-deviation}).
    We consider an example of the Tucker-2 model with a loss function $\mathbb E[\|AGB - (T+\alpha N)\|^2_F]$,
    where $T$ is the target tensor, $N$ is the Gaussian noise, and $\alpha$ controls the strength
    of noise. Left: Norm Deviation $Q$ vs. iterations. Right: $\mathrm{Cov}(\|\mathcal G_k\|_F^2, \|g_k\|_F^2)$
    vs. iterations.
}
    \label{fig:tucker}
\end{figure}
In contrast, positive covariance means that large cores will grow faster, 
leading to larger imbalance among the cores. 
\paragraph{Toward norm balancing tendencies under SAM.}
Our theoretical analysis reveals that SAM exhibits strong norm-balancing tendencies at a local level. Specifically, the pairwise norm difference, \textit{i.e.}, $\left|\|\mathcal G_i\|_F^2 - \|\mathcal G_j\|_F^2\right|, \forall i,j \in [K]$, shrinks \textit{locally} when being large (Proposition~\ref{prop:pairwise-norm-shrinkage} in Appendix). 
This local corrective dynamic for every pair suggests a global effect; one can intuit from Equation~\eqref{eq:sqsum-pairwise} that an accumulation of these pairwise shrinkages would lead to an overall decrease in the Norm Deviation $Q$.
While our proof formally establishes this balancing effect at a local level, our experiments confirm its global impact. Empirical observations consistently show that $Q$ decreases during training (see Fig.~\ref{fig:tucker}), supporting the conjecture that SAM promotes norm balancing through an accumulation of these local shrinkage effects.
In the controlled Tucker-2 experiment of Fig.~\ref{fig:tucker}, we introduce additive Gaussian noise with varying magnitude controlled by a scalar $\alpha$. As $\alpha$ increases, the gradients become noisier, and the covariance between core norms and their corresponding gradient norms grows. This larger covariance corresponds to a faster reduction in $Q$, reflecting a stronger balancing effect under SAM. This phenomenon illustrates how the implicit regularization strength of SAM adapts to the noise level, aligning with the ``data-responsive'' interpretation observed in the two-factor setting of~\cite{li2024implicit}. The benefits of norm balancing in gradient-based optimization are well-documented for both tensor decomposition~\cite{du2018algorithmic,razin2022implicit,hariz2022implicit} and neural networks~\cite{neyshabur2017exploring}, suggesting that the norm-regulating dynamics induced by SAM are potentially favorable for optimization.

\paragraph{Extension to multi-layer models.} Our analysis naturally extends models that consist of multiple scale-invariant models across layers, such as tensorized neural networks, SAM governs the Norm Deviation layer-wisely. See Theorem~\ref{thm:multi-layer-sam-norm-deviation} in the Appendix as a multi-layer extension of Theorem~\ref{thm:sam-norm-deviation}.

\section{Mimicking the Implicit Regularization of SAM with Explicit Control}
The theoretical insights suggest that SAM induces a tendency to control the core norms, particularly when the covariance between core and gradient magnitudes is high. Prior works turn or enhance implicit effects of optimization with explicit regularizers, such as \cite{wei2020implicit,barrett2021implicit,li2024implicit,li2024efficiency}. Motivated by this, this section aims to address the following question: \textit{\textbf{Can we design a method that mimics the implicit regularization to utilize its beneficial impact?}} By an explicit method that replicates the implicit regularization of SAM, it is possible to \textbf{(1)} reduce the extra gradient calculation of SAM (line 2 in Algorithm~\ref{alg:sam-cores}) if we use a clever design, and \textbf{(2)} decouple the norm control from the optimization process, which allows us to control the norm of cores independently of the optimization process.

\subsection{Deviation-Aware Scaling}
In this section, we gain insights from theory and
 introduce a novel method \texttt{Deviation-Aware Scaling} (DAS), which
address the above question by explicitly controlling the Norm Deviation $Q$
through a scaling approach.
We adopt a simple exemplar that scales the cores by a factor in each iteration to control norms, following the implementation of weight decay~\cite{loshchilov2018decoupled}:
\begin{align}
    \mathcal G_k^{(t+\frac{1}{2})} &= (1+\lambda_k^{(t)}) \mathcal G_k^{(t)}, \label{eq:scale-core-a} \\
    \mathcal G_k^{(t+1)} &= \mathcal G_k^{(t+\frac{1}{2})} - \eta g_k^{(t)}, \label{eq:scale-core-b}
\end{align}
for each core $k\in [K]$ at iteration $t$, where $\eta$ is a finite stepsize and $\lambda_k^{(t)}\in \mathbb R$ for core $k$ at iteration $t$.
Based on Corollary~\ref{cor:sgd-norm-deviation}, we can assume only the scaling step \eqref{eq:scale-core-a} controls the Norm Deviation $Q$.
We can derive a closed-form expression for the scaling factor $\lambda_k^{(t)}$ to control the Norm Deviation $Q$
to match the implicit regularization of SAM using only SGD.
\paragraph{Deriving $\lambda_k$.} 
With small $\lambda_k^{(t)}$, the change in norm caused by the scaling step \eqref{eq:scale-core-a} is:
\begin{align*}
\Delta \|\mathcal G_k\|^2_F &:= \|\mathcal G_k^{(t+\frac{1}{2})}\|^2_F - \|\mathcal G_k^{(t)}\|^2_F\approx 2\lambda_k^{(t)} \|\mathcal G_k^{(t)}\|^2_F.
\end{align*}
We would like to set the $\lambda_k^{(t)}$ so that the dynamics in the Norm Deviation
$Q$ caused by the scaling step \eqref{eq:scale-core-a} matches the dynamics of SAM in Theorem~\ref{thm:sam-norm-deviation}.
With a small enough stepsize $\eta$, the change in $Q$ under SAM is approximately:
\begin{align*}
\Delta Q_{\mathrm{SAM}} &\approx \eta \cdot \left(\frac{dQ}{dt}\right)_{\mathrm{SAM}}\\
& \approx 4\eta \rho u^{(t)} K \cdot \mathrm{Cov}\left(\|\mathcal G_k^{(t)}\|_F^2, \|g_k^{(t)}\|_F^2\right),
\end{align*}
where the term $O(\rho^2 L)$ is ignored for small $\rho$.
At the same time, the change in $Q$ caused by the scaling step \eqref{eq:scale-core-a} is:
\begin{align*}
\Delta Q &\approx  4\sum_{k=1}^K \left(\|\mathcal G_k^{(t)}\|_F^2 - \frac{1}{K}\sum_{i=1}^K \|\mathcal G_i^{(t)}\|_F^2\right) \cdot \lambda_k^{(t)} \|\mathcal G_k^{(t)}\|^2_F,
\end{align*}
following the definition of $Q$ (Definition~\ref{def:global-norm-deviation}). 
We would like to match the two changes in $Q$, \textit{i.e.}, $\Delta Q = \Delta Q_{\mathrm{SAM}}$.
This leads to the following closed-form expression for $\lambda_k^{(t)}$:
\begin{equation}
\label{eq:lambda-eq}
\lambda_k^{(t)} = \frac{\rho u^{(t)}\cdot \eta}{\|\mathcal G_k^{(t)}\|_F^2} \cdot \left(\|g_k^{(t)}\|_F^2 - \bar{g}\right),
\end{equation}
where $\bar{g} = \frac{1}{K}\sum_{i=1}^K \|g_i^{(t)}\|_F^2$ is the average squared gradient norm over all cores at iteration $t$.
However, since this expression is derived from approximations of the true dynamics, there is no guarantee that SAM's original perturbation radius $\rho$ is the optimal choice for controlling the strength of this explicit scaling.
We therefore replace $\rho$ with a new hyperparameter, $\alpha$, which directly controls the strength of the scaling. This decouples our method from SAM's settings and leads to the final Deviation-Aware Scaling (DAS) algorithm, summarized in Algorithm \ref{alg:scale-cores}.

\begin{algorithm}[tb]
\caption{Deviation-Aware Scaling (DAS)}
\label{alg:scale-cores}
\begin{algorithmic}[1]
\STATE \textbf {Initialize:} tensor cores $\{\mathcal{G}_k^{(0)}\}$, stepsize $\{\eta^{(t)}\}$, coefficient $\{\alpha^{(t)}\}$, and number of iterations $T$.
\FOR{$t = 0, \ldots, T-1$}
    \STATE Compute $g_k^{(t)} = \nabla_{\mathcal{G}_k} f(\Phi(\mathcal{G}_1^{(t)}, \ldots, \mathcal{G}_K^{(t)})), \forall k \in [K]$
    \FOR{$k = 1, \ldots, K$}
    \STATE $\lambda_k^{(t)} = \frac{\eta^{(t)} \alpha^{(t)} u^{(t)}}{\|\mathcal{G}_k^{(t)}\|_F^2} \cdot \left(\|g_k^{(t)}\|_F^2 - \bar{g}\right)$, \\ where $\bar{g} = \frac{1}{K}\sum_{i=1}^K \|g_{i}^{(t)}\|_F^2$, $u^{(t)} = (K\cdot\bar g)^{-1/2}$
    \STATE Update $\mathcal{G}_k^{(t+1)} = (1+\lambda_k^{(t)}) \mathcal{G}_k^{(t)} - \eta g_k^{(t)}$\\ via Adam or SGD
    \ENDFOR
\ENDFOR
\end{algorithmic}
\end{algorithm}

\section{Experiments}
We conduct a comprehensive set of experiments to show the effectiveness of SAM and DAS for various tensor-based models, including tensor completion, training tensorized neural networks, and improving tensor-based parameter-efficient fine-tuning (PEFT) for large language model adaptation. See Appendix for more experiment details.

\begin{figure}[t]
\centering
\includegraphics[width=\columnwidth]{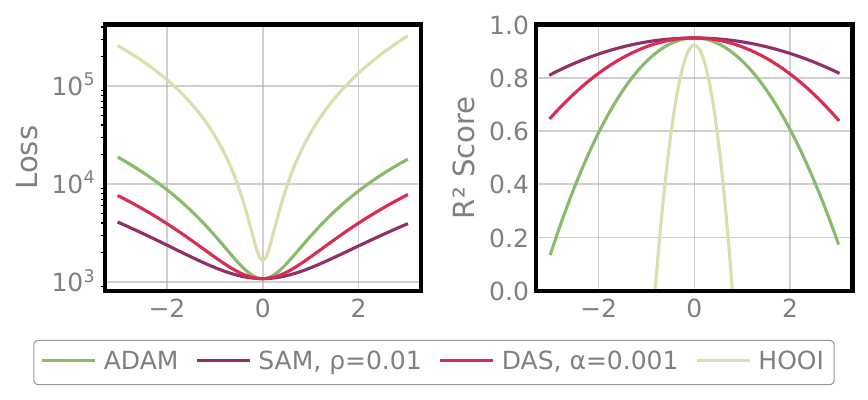}
\caption{
    1-D visualization of train loss and evaluation metric vs. perturbations on the COVID dataset. Left: Training loss. Right: Evaluation $\mathrm{R}^2$ score. The x-axis is the size of a fixed directional perturbation applied to model parameters.
}
    \label{fig:tucker-covid-loss}
\end{figure}
\begin{table}[t]
\centering
\begin{tabular}{l|cccc}
    \toprule
    Method & HOOI & ADAM & SAM & DAS \\
    \midrule
    $\mathrm{R}^2$ score & $0.9268$ & $0.9482$ & $0.9485^{**}$ & $0.9484^{*}$ \\
    \bottomrule
\end{tabular}
\caption{Results of Tucker on COVID dataset. Best and second-best results per row are marked with $^{**}$ and $^{*}$, respectively. This notation applies to Tables \ref{tab:tucker-covid}-\ref{tab:tt-resnet-finetune}.}
\label{tab:tucker-covid}
\end{table}
\subsection{Tucker Decomposition for Tensor Completion}
We perform experiments on the Tucker decomposition for tensor completion, using the real-world data COVID dataset from the library tensorly~\cite{kossaifi2019tensorly}. 
We randomly mask 70\% of the entries as the evaluation set to be recovered, and use the remaining 30\% as the training set.
We follow the setup of \cite{hariz2024implicit} and use a three-order Tucker with multilinear rank $(8,6,8)$ and an MSE loss function optimized with ADAM~\cite{kingma2015adam}.

We use SAM and DAS with a base optimizer ADAM to compare with the vanilla ADAM. 
Tucker-HOOI~\cite{kolda2009tensor} is also included as a baseline, available in tensorly.
The results averaged from 3 dataset splits are shown in Table~\ref{tab:tucker-covid} and Fig.~\ref{fig:tucker-covid-loss}. In the table, we report the coefficient of determination $\mathrm{R}^2$ score, which measures the proportion of variance explained by the model.
SAM and DAS show marginal improvements over ADAM, while all three gradient-based methods substantially outperform the traditional HOOI baseline.
The loss curves in Fig.~\ref{fig:tucker-covid-loss} reveal an important insight: while SAM and DAS achieve only marginal improvements in the final $\mathrm{R}^2$ score, they find flatter minima compared to vanilla ADAM, which is precisely the intended effect of SAM's sharpness-aware optimization. Importantly, DAS successfully captures this flatness property through explicit norm control rather than SAM's direct perturbation-based approach, validating our theoretical analysis that the norm dynamics is a key driver of SAM's beneficial effects.

\subsection{Applications to Tensorized Neural Networks}
Our setting of general scale-invariant problems applies to 
a wide range of multi-layer tensorized neural networks.

\paragraph{Training Tensorized Neural Networks from Scratch.}
We evaluate the effectiveness of SAM and DAS on CIFAR-10~\cite{krizhevsky2009learning} using ResNet-20~\cite{he2016deep} parameterized by tensor decompositions. Each convolution layer is replaced with a set of tensor cores, which reconstruct the full weight tensor. We experiment with both CP~\cite{kolda2009tensor} and Tensor-Ring (TR)~\cite{zhao2016tensor} decompositions, with five different ranks for each method.
The models are trained from scratch using SGD, SAM, and DAS (using SGD as the base optimizer). For both SAM and DAS, the corresponding hyperparameters $\rho$ and $\alpha$ are tuned independently over the shared search space $\{0.001, 0.005, 0.01, 0.05, 0.1\}$ using a validation split from the training set.

Figure~\ref{fig:tresnet-cifar10} reports the mean and standard error of accuracy improvements (over SGD) from 10 independent runs, after training for 180 epochs. Both SAM and DAS consistently outperform SGD across all decomposition types and ranks. Interestingly, the performance gains from SAM tend to increase with higher tensor ranks, particularly in the Tensor-Ring setting, suggesting that SAM is especially beneficial for larger tensorized models. DAS also provides consistent, though slightly smaller, improvements across configurations. This highlights that DAS explicitly captures and benefits from the norm dynamics that underpin SAM’s implicit regularization, offering an efficient alternative.

\begin{figure}
    \centering
    \includegraphics[width=3.3in]{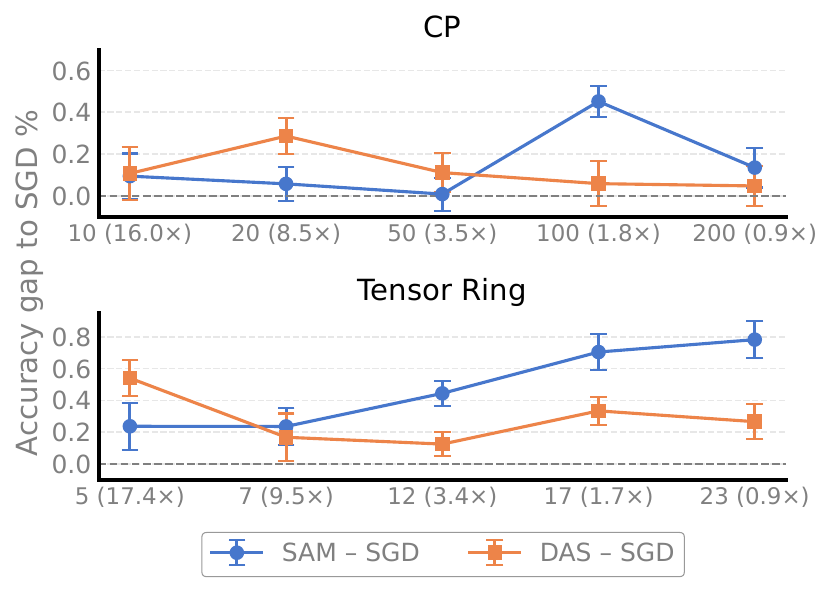}
    \caption{
Accuracy improvements over SGD when training tensorized ResNet-20 on CIFAR-10 with CP (top) and Tensor-Ring (bottom) decompositions. The x-axis indicates the chosen rank for each decomposition, with the corresponding model compression ratio shown in parentheses (\textit{e.g.}, 
$a\times$ denotes the ratio of original to compressed model parameters). Error bars indicate standard error over 10 runs.
}
    \label{fig:tresnet-cifar10}
\end{figure}

\paragraph{Tensorized Neural Networks under Label Noise.}
It is shown in \cite{foret2020sharpness,kwon2021asam} that SAM is robust to label noise
on general deep models, without specific designs for label noise.
We test the robustness of SAM and DAS on a tensorized neural network, using a compact version of ResNet-32, where all convolution layers are re-parameterized using Tensor-Ring decomposition with ranks set as in Appendix.
We use a symmetric label noise~\cite{van2015learning,jiang2020beyond} with corruption rate $40\%, 60\%, 80\%$ on CIFAR-10 training set and a clean test set. Table~\ref{tab:tr-resnet-label-noise} shows the test accuracies of TR-ResNet-32 on CIFAR-10 with label noise. 
We compare the performance of SGD, SAM, and DAS (with SGD as the base optimizer).
As shown in Table~\ref{tab:tr-resnet-label-noise}, both SAM and DAS improve robustness to label noise over SGD, with SAM achieving significantly better test accuracy under higher noise levels. While DAS does not match SAM's robustness, it still outperforms SGD, indicating partial resistance to label corruption. Prior work~\cite{baek2024why} attributes SAM's robustness to its influence mainly on network Jacobian effects that DAS does not replicate. This suggests that DAS may derive its robustness from its explicit control over parameter norm dynamics, leading to better optimization behavior than SGD.

\begin{table}[t]
\centering
\begin{tabular}{lccc}
    \toprule
    Noise rate& SGD & SAM & DAS \\
    \midrule
    40\% & $83.76_{\pm 0.24}$ & $86.59^{**}_{\pm 0.16}$ & $84.35^*_{\pm 0.40}$ \\
    60\% & $77.41_{\pm 1.08}$ & $82.03^{**}_{\pm 0.25}$ & $77.73^*_{\pm 0.27}$ \\
    80\% & $57.33_{\pm 0.44}$ & $67.13^{**}_{\pm 1.70}$ & $60.74^*_{\pm 0.57}$ \\
    \midrule
    Runtime (s) & $0.039^{**}$ & $0.090$ & $0.054^*$  \\
    \bottomrule
\end{tabular}
\caption{Test accuracies and runtime for SGD, SAM, and DAS on CIFAR-10 with label noise using TR-ResNet-32. Means and standard deviations are computed over 3 runs.}
\label{tab:tr-resnet-label-noise}
\end{table}

\paragraph{Finetuning Pre-trained Models after Compression.}

We consider a practical scenario where a pre-trained model is compressed using tensor decomposition and then fine-tuned to restore the performance lost due to compression, following a paradigm similar to \cite{phan2020stable}. 
We compress all convolution layers in a pre-trained ResNet-18 using Tensor-Train (TT) decomposition, with ranks suggested in \cite{yin2021towards}. This results in a TT-ResNet-18 with 4.4M parameters, compared to the original ResNet-18 with 11.2M parameters.
The compressed TT-ResNet-18 is fine-tuned on the ImageNet-1k~\cite{deng2009imagenet}.
We compare the performance of SGD, SAM, and DAS (with SGD as the base optimizer). We consider a lightweight fine-tuning with $15$ epochs.
Table~\ref{tab:tt-resnet-finetune} shows the mean and standard deviation of Top-1 and Top-5 accuracies over 5 runs. Note that the compressed model before fine-tuning achieves $36.74\%$/$64.32\%$ Top-1/-5 accuracy, which is significantly lower than the original ResNet-18. We observe that both SAM and DAS outperform SGD, with SAM achieving the best performance. DAS uses approximately the same runtime as SGD, but its performance is close to SAM, showing its effectiveness.

\begin{table}[t]
\centering
\begin{tabular}{lccc}
    \toprule
    & SGD & SAM & DAS \\
    \midrule
    Top-1 & $65.47_{\pm 0.14}$ & $66.27^{**}_{\pm 0.07}$ & $66.16^{*}_{\pm 0.21}$ \\
    Top-5 & $86.54_{\pm 0.14}$ & $87.12^{**}_{\pm 0.05}$ & $86.96^{*}_{\pm 0.05}$ \\    
    \midrule
    Runtime (s) & $0.254^{**}$ & $0.425^{*}$ & $0.254^{**}$ \\
    \bottomrule
\end{tabular}
\caption{Top1, Top5, and runtime for SGD, SAM, and DAS fine-tuned on ImageNet using compressed TT-ResNet-18.}
\label{tab:tt-resnet-finetune}
\end{table}

\begin{table*}[t]
\centering
\begin{tabular}{ll|llllll|l}
\toprule
\textbf{RoBERTa} & & \textbf{SST-2} & \textbf{SST-5} & \textbf{SNLI} & \textbf{MNLI} & \textbf{RTE} & \textbf{TREC} & \textbf{Avg.(↑)} \\
\midrule
\multicolumn{2}{l|}{Zero-Shot$^\dag$} &
79.0 &
35.5 &
50.2 &
48.8 &
51.4 &
32.0 &
49.5 \\
\midrule
\multicolumn{2}{l|}{Full fine-tuning} &
94.01$^{**}_{\pm 0.46}$ &
56.84$^{**}_{\pm 0.86}$ &
88.38$^{**}_{\pm 0.43}$ &
84.32$^{**}_{\pm 0.72}$ &
83.16$_{\pm 1.67}$ &
96.92$^{*}_{\pm 0.55}$ &
83.94$^{*}$ \\
\multicolumn{2}{l|}{LoRA ($r=8$)} &
90.53$_{\pm 1.10}$ &
51.60$_{\pm 0.77}$ &
83.58$_{\pm 1.24}$ &
76.38$_{\pm 2.99}$ &
77.55$_{\pm 4.32}$ &
95.72$_{\pm 0.41}$ &
79.23 \\
\midrule
FLoRA & ADAM &
93.99$^{*}_{\pm 0.26}$ &
56.70$_{\pm 0.80}$ &
87.70$_{\pm 0.54}$ &
83.96$_{\pm 0.67}$ &
83.61$_{\pm 1.30}$ &
96.64$_{\pm 0.34}$ &
83.78 \\
($r=8$) & SAM &
93.76$_{\pm 0.48}$ &
56.76$_{\pm 1.39}$ &
88.16$^{*}_{\pm 0.67}$ &
83.56$_{\pm 0.73}$ &
83.68$^{*}_{\pm 1.62}$ &
97.08$^{**}_{\pm 0.27}$ &
83.83 \\
& DAS &
93.94$_{\pm 0.56}$ &
56.82$^{*}_{\pm 0.90}$ &
87.90$_{\pm 0.57}$ &
84.28$^{*}_{\pm 0.63}$ &
83.97$^{**}_{\pm 0.87}$ &
96.80$_{\pm 0.42}$ &
83.95$^{**}$ \\
\bottomrule
\end{tabular}
\caption{Results on RoBERTa-large fine-tuned on GLUE in the low-data regime across 5 runs with different sampled data. Results marked with $\dag$ are reported by~\cite{malladi2023fine}. Best and second-best results per column are marked with $^{**}$ and $^{*}$, respectively. This notation applies to Tables \ref{tab:roberta-result} and \ref{tab:loretta-results}.}
\label{tab:roberta-result}
\end{table*}

\begin{table*}[t]
\centering
\begin{tabular}{ll|r|lllllll|l}
\toprule
\multicolumn{2}{l|}{\textbf{OPT-6.7B}} & \textbf{Params} & \textbf{CB} & \textbf{BoolQ} & \textbf{WSC} & \textbf{COPA} & \textbf{ReCoRD} & \textbf{SQuAD} & \textbf{DROP} & \textbf{Avg.(↑)} \\
\midrule
\multicolumn{2}{l|}{Zero-Shot}       &  & 60.7 & 65.9 & 37.5 & 80 & 76.9 & 69.5 & 26.4 & 59.56 \\
\midrule
\multicolumn{2}{l|}{Full fine-tuning}& 6658.47M & 71.4 & 68.7 & 63.5$^{**}$ & 82$^{**}$ & 78.7$^{**}$ & 81.8 & 29.2 & 67.89 \\
\multicolumn{2}{l|}{LoRA ($r=16$)}     & 8.39M & 87.5$^{**}$ & 77.8$^{*}$ & 63.5$^{**}$ & 81$^{*}$ & 77.1 & 85.9 & 32.7$^{**}$ & 72.21$^{**}$ \\
\midrule
LoRETTA & ADAM & \multirow{3}{*}{0.96M} & 80.4 & 76.6 & 58.7 & 80 & 77.3$^{*}$ & 86.7$^{*}$ & 32.1 & 70.25 \\
  ($r=16$)     & SAM  &                         & 85.7$^{*}$ & 78.6$^{**}$ & 63.5$^{**}$ & 77 & 76.8 & 88.7$^{**}$ & 31.8 & 71.72$^{*}$ \\
        & DAS  &                         & 82.1 & 77.2 & 63.5$^{**}$ & 81$^{*}$ & 77.2 & 86.2 & 32.7$^{**}$ & 71.41 \\
\bottomrule
\end{tabular}
\caption{Results on OPT-6.7B fine-tuned on SuperGLUE and generation tasks in the low-data regime.}
\label{tab:loretta-results}
\end{table*}

\begin{table}[t]
\centering
\begin{tabular}{lccc}
    \toprule
    & ADAM & SAM & DAS \\
    \midrule
    Runtime ($\downarrow$) & $1\times$ & $2\times$ & $1.04\times$ \\
    \bottomrule
\end{tabular}
\caption{Normalized runtime on OPT-6.7B using LoRETTA.}
\label{tab:loretta-runtime}
\end{table}

\subsection{Tensor-based Parameter-Efficient Fine-Tuning}
We evaluate SAM and DAS on two recent tensor-based low-rank adaptation methods, FLoRA~\cite{si2025maintaining} and LoRETTA~\cite{yang2024loretta} for fine-tuning language models. See Appendix for details of FLoRA and LoRETTA.

\paragraph{Improving FLoRA on RoBERTa-large.}
FLoRA utilizes Tucker decomposition to parameterize the incremental update for low-rank adaptation.
We fine-tune RoBERTa-large, a pre-trained language model with 355M parameters, following the setting in~\cite{malladi2023fine}. We use a challenging few-shot learning setting, sampling 512 examples per class. Results are summarized in Table~\ref{tab:roberta-result}. Both SAM and DAS improve over the FLoRA baseline. Notably, DAS slightly outperforms both SAM and full fine-tuning. We hypothesize that this performance gap arises because SAM, despite its sharpness-minimization objective, may exhibit unintended side effects in low-rank subspace adaptation—such as converging to sharp minima due to perturbations outside the update subspace~\cite{liflat}. In contrast, DAS isolates and distills the beneficial norm dynamics of SAM while avoiding such drawbacks.

\paragraph{Improving LoRETTA on OPT-6.7B.}

LoRETTA utilizes TT decomposition to parameterize the incremental update for low-rank adaptation. 
We fine-tune OPT-6.7B~\cite{zhang2022opt}, an autoregressive language model with 6.7B parameters on the SuperGLUE tasks~\cite{wang2019superglue} and
generation tasks including SQuAD~\cite{rajpurkar2016squad} and DROP~\cite{dua2019drop} using LoRETTA.
We follow the setup in \cite{yang2024loretta} to use a challenging low-data setting with 1000/500/1000 examples for training/validation/testing, using a prompt-based fine-tuning suggested in~\cite{malladi2023fine}.
Results of SAM and DAS on OPT-6.7B are summarized in Table~\ref{tab:loretta-results}. As shown in the zero-shot and full fine-tuning results, the distribution shift between pre-trained data and fine-tuned data is large and can be overfitted by full fine-tuning. SAM improves LoRETTA by a significant margin, and DAS achieves a similar improvement.
We also compare against LoRA~\cite{hu2022lora} with rank 16, using approximately $8\times$ trainable parameters of LoRETTA. SAM and DAS enhance the much lighter LoRETTA to achieve a more competitive performance compared to LoRA, demonstrating the effectiveness of the proposed methods. Moreover, we compared the runtime of ADAM, SAM, and DAS. In Table~\ref{tab:loretta-runtime}, DAS saves more than $90\%$ runtime of SAM but still achieves a competitive performance compared to SAM, suggesting scaling as a strong yet efficient alternative for the extra gradient of SAM.

\section{Conclusion}
In this work, we investigated the implicit regularization of SAM in general, multi-core tensorized models. Our theoretical analysis reveals that a key mechanism of SAM is the norm dynamics, governed by the covariance between the norms of tensor cores and gradient magnitudes. We distilled this insight into a simple yet effective method, DAS, explicitly mimicking this regularization. Extensive experiments on tasks including tensor completion, model compression, and parameter-efficient fine-tuning validate the effectiveness of SAM and DAS as a computationally efficient alternative.

\bibliography{aaai2026}

\begin{thebibliography}{75}
\providecommand{\natexlab}[1]{#1}

\bibitem[{Akiba et~al.(2019)Akiba, Sano, Yanase, Ohta, and
  Koyama}]{akiba2019optuna}
Akiba, T.; Sano, S.; Yanase, T.; Ohta, T.; and Koyama, M. 2019.
\newblock Optuna: A next-generation hyperparameter optimization framework.
\newblock In \emph{Proceedings of the 25th ACM SIGKDD International Conference
  on Knowledge Discovery \& Data Mining}, 2623--2631.

\bibitem[{Andriushchenko et~al.(2023)Andriushchenko, Croce, M{\"u}ller, Hein,
  and Flammarion}]{andriushchenko2023modern}
Andriushchenko, M.; Croce, F.; M{\"u}ller, M.; Hein, M.; and Flammarion, N.
  2023.
\newblock A modern look at the relationship between sharpness and
  generalization.
\newblock \emph{arXiv preprint arXiv:2302.07011}.

\bibitem[{Andriushchenko and Flammarion(2022)}]{andriushchenko2022towards}
Andriushchenko, M.; and Flammarion, N. 2022.
\newblock Towards understanding sharpness-aware minimization.
\newblock In \emph{International Conference on Machine Learning}, 639--668.
  PMLR.

\bibitem[{Arora et~al.(2019)Arora, Cohen, Hu, and Luo}]{arora2019implicit}
Arora, S.; Cohen, N.; Hu, W.; and Luo, Y. 2019.
\newblock Implicit regularization in deep matrix factorization.
\newblock \emph{Advances in Neural Information Processing Systems}, 32.

\bibitem[{Baek, Kolter, and Raghunathan(2024)}]{baek2024why}
Baek, C.; Kolter, J.~Z.; and Raghunathan, A. 2024.
\newblock Why is {SAM} Robust to Label Noise?
\newblock In \emph{The Twelfth International Conference on Learning
  Representations}.

\bibitem[{Bahri, Mobahi, and Tay(2022)}]{bahri2022sharpness}
Bahri, D.; Mobahi, H.; and Tay, Y. 2022.
\newblock Sharpness-Aware Minimization Improves Language Model Generalization.
\newblock In \emph{Proceedings of the 60th Annual Meeting of the Association
  for Computational Linguistics (Volume 1: Long Papers)}, 7360--7371.

\bibitem[{Barrett and Dherin(2021)}]{barrett2021implicit}
Barrett, D.; and Dherin, B. 2021.
\newblock Implicit Gradient Regularization.
\newblock In \emph{International Conference on Learning Representations}.

\bibitem[{Bartlett, Long, and Bousquet(2023)}]{bartlett2023dynamics}
Bartlett, P.~L.; Long, P.~M.; and Bousquet, O. 2023.
\newblock The dynamics of sharpness-aware minimization: Bouncing across ravines
  and drifting towards wide minima.
\newblock \emph{Journal of Machine Learning Research}, 24(316): 1--36.

\bibitem[{Belkin et~al.(2019)Belkin, Hsu, Ma, and
  Mandal}]{belkin2019reconciling}
Belkin, M.; Hsu, D.; Ma, S.; and Mandal, S. 2019.
\newblock Reconciling modern machine-learning practice and the classical
  bias--variance trade-off.
\newblock \emph{Proceedings of the National Academy of Sciences}, 116(32):
  15849--15854.

\bibitem[{Bisla, Wang, and Choromanska(2022)}]{bisla2022low}
Bisla, D.; Wang, J.; and Choromanska, A. 2022.
\newblock Low-pass filtering sgd for recovering flat optima in the deep
  learning optimization landscape.
\newblock In \emph{International Conference on Artificial Intelligence and
  Statistics}, 8299--8339. PMLR.

\bibitem[{Cao et~al.(2024)Cao, Sun, Nguyen, and Mamitsuka}]{cao2024learning}
Cao, T.; Sun, L.; Nguyen, C.~H.; and Mamitsuka, H. 2024.
\newblock Learning low-rank tensor cores with probabilistic L0-regularized rank
  selection for model compression.
\newblock In \emph{Proceedings of the Thirty-Third International Joint
  Conference on Artificial Intelligence}, 3780--3788.

\bibitem[{Compagnoni et~al.(2023)Compagnoni, Biggio, Orvieto, Proske, Kersting,
  and Lucchi}]{compagnoni2023sde}
Compagnoni, E.~M.; Biggio, L.; Orvieto, A.; Proske, F.~N.; Kersting, H.; and
  Lucchi, A. 2023.
\newblock An sde for modeling sam: Theory and insights.
\newblock In \emph{International Conference on Machine Learning}, 25209--25253.
  PMLR.

\bibitem[{Deng et~al.(2009)Deng, Dong, Socher, Li, Li, and
  Fei-Fei}]{deng2009imagenet}
Deng, J.; Dong, W.; Socher, R.; Li, L.-J.; Li, K.; and Fei-Fei, L. 2009.
\newblock Imagenet: A large-scale hierarchical image database.
\newblock In \emph{2009 IEEE Conference on Computer Vision and Pattern
  Recognition}, 248--255. Ieee.

\bibitem[{Deng et~al.(2025)Deng, Pang, Zhang, and Guo}]{deng2025asymptotic}
Deng, J.; Pang, J.; Zhang, B.; and Guo, G. 2025.
\newblock Asymptotic Unbiased Sample Sampling to Speed Up Sharpness-Aware
  Minimization.
\newblock In \emph{Proceedings of the AAAI Conference on Artificial
  Intelligence}, volume~39, 16208--16216.

\bibitem[{Du et~al.(2021)Du, Yan, Feng, Zhou, Zhen, Goh, and
  Tan}]{du2021efficient}
Du, J.; Yan, H.; Feng, J.; Zhou, J.~T.; Zhen, L.; Goh, R. S.~M.; and Tan, V.~Y.
  2021.
\newblock Efficient sharpness-aware minimization for improved training of
  neural networks.
\newblock \emph{arXiv preprint arXiv:2110.03141}.

\bibitem[{Du et~al.(2022)Du, Zhou, Feng, Tan, and Zhou}]{du2022sharpness}
Du, J.; Zhou, D.; Feng, J.; Tan, V.; and Zhou, J.~T. 2022.
\newblock Sharpness-aware training for free.
\newblock \emph{Advances in Neural Information Processing Systems}, 35:
  23439--23451.

\bibitem[{Du, Hu, and Lee(2018)}]{du2018algorithmic}
Du, S.~S.; Hu, W.; and Lee, J.~D. 2018.
\newblock Algorithmic regularization in learning deep homogeneous models:
  Layers are automatically balanced.
\newblock \emph{Advances in Neural Information Processing Systems}, 31.

\bibitem[{Dua et~al.(2019)Dua, Wang, Dasigi, Stanovsky, Singh, and
  Gardner}]{dua2019drop}
Dua, D.; Wang, Y.; Dasigi, P.; Stanovsky, G.; Singh, S.; and Gardner, M. 2019.
\newblock DROP: A Reading Comprehension Benchmark Requiring Discrete Reasoning
  Over Paragraphs.
\newblock In \emph{Proceedings of the 2019 Conference of the North American
  Chapter of the Association for Computational Linguistics: Human Language
  Technologies, Volume 1 (Long and Short Papers)}, 2368--2378.

\bibitem[{Foret et~al.(2020)Foret, Kleiner, Mobahi, and
  Neyshabur}]{foret2020sharpness}
Foret, P.; Kleiner, A.; Mobahi, H.; and Neyshabur, B. 2020.
\newblock Sharpness-aware minimization for efficiently improving
  generalization.
\newblock \emph{arXiv preprint arXiv:2010.01412}.

\bibitem[{Ge et~al.(2015)Ge, Huang, Jin, and Yuan}]{ge2015escaping}
Ge, R.; Huang, F.; Jin, C.; and Yuan, Y. 2015.
\newblock Escaping from saddle points—online stochastic gradient for tensor
  decomposition.
\newblock In \emph{Conference on learning theory}, 797--842. PMLR.

\bibitem[{Gunasekar et~al.(2017)Gunasekar, Woodworth, Bhojanapalli, Neyshabur,
  and Srebro}]{gunasekar2017implicit}
Gunasekar, S.; Woodworth, B.~E.; Bhojanapalli, S.; Neyshabur, B.; and Srebro,
  N. 2017.
\newblock Implicit regularization in matrix factorization.
\newblock \emph{Advances in Neural Information Processing Systems}, 30.

\bibitem[{Hariz et~al.(2022)Hariz, Kadri, Ayache, Moakher, and
  Artieres}]{hariz2022implicit}
Hariz, K.; Kadri, H.; Ayache, S.; Moakher, M.; and Artieres, T. 2022.
\newblock Implicit regularization with polynomial growth in deep tensor
  factorization.
\newblock In \emph{International Conference on Machine Learning}, 8484--8501.
  PMLR.

\bibitem[{Hariz et~al.(2024)Hariz, Kadri, Ayache, Moakher, and
  Arti{\`e}res}]{hariz2024implicit}
Hariz, K.; Kadri, H.; Ayache, S.; Moakher, M.; and Arti{\`e}res, T. 2024.
\newblock Implicit regularization in deep tucker factorization: Low-rankness
  via structured sparsity.
\newblock In \emph{International Conference on Artificial Intelligence and
  Statistics}, 2359--2367. PMLR.

\bibitem[{Hayashi et~al.(2019)Hayashi, Yamaguchi, Sugawara, and
  Maeda}]{hayashi2019exploring}
Hayashi, K.; Yamaguchi, T.; Sugawara, Y.; and Maeda, S.-i. 2019.
\newblock Exploring unexplored tensor network decompositions for convolutional
  neural networks.
\newblock \emph{Advances in Neural Information Processing Systems}, 32.

\bibitem[{He et~al.(2016)He, Zhang, Ren, and Sun}]{he2016deep}
He, K.; Zhang, X.; Ren, S.; and Sun, J. 2016.
\newblock Deep residual learning for image recognition.
\newblock In \emph{Proceedings of the IEEE Conference on Computer Vision and
  Pattern Recognition}, 770--778.

\bibitem[{Hrinchuk et~al.(2020)Hrinchuk, Khrulkov, Mirvakhabova, Orlova, and
  Oseledets}]{hrinchuk2020tensorized}
Hrinchuk, O.; Khrulkov, V.; Mirvakhabova, L.; Orlova, E.; and Oseledets, I.
  2020.
\newblock Tensorized Embedding Layers.
\newblock In \emph{Findings of the Association for Computational Linguistics:
  EMNLP 2020}, 4847--4860.

\bibitem[{Hu et~al.(2022)Hu, yelong shen, Wallis, Allen-Zhu, Li, Wang, Wang,
  and Chen}]{hu2022lora}
Hu, E.~J.; yelong shen; Wallis, P.; Allen-Zhu, Z.; Li, Y.; Wang, S.; Wang, L.;
  and Chen, W. 2022.
\newblock Lo{RA}: Low-Rank Adaptation of Large Language Models.
\newblock In \emph{International Conference on Learning Representations}.

\bibitem[{Ilbert et~al.(2024)Ilbert, Odonnat, Feofanov, Virmaux, Paolo,
  Palpanas, and Redko}]{ilbert2024samformer}
Ilbert, R.; Odonnat, A.; Feofanov, V.; Virmaux, A.; Paolo, G.; Palpanas, T.;
  and Redko, I. 2024.
\newblock SAMformer: Unlocking the Potential of Transformers in Time Series
  Forecasting with Sharpness-Aware Minimization and Channel-Wise Attention.
\newblock In \emph{International Conference on Machine Learning}, 20924--20954.
  PMLR.

\bibitem[{Ishida et~al.(2020)Ishida, Yamane, Sakai, Niu, and
  Sugiyama}]{ishida2020we}
Ishida, T.; Yamane, I.; Sakai, T.; Niu, G.; and Sugiyama, M. 2020.
\newblock Do We Need Zero Training Loss After Achieving Zero Training Error?
\newblock In \emph{International Conference on Machine Learning}, 4604--4614.
  PMLR.

\bibitem[{Ji et~al.(2024)Ji, Li, Fu, Afghah, Guo, Yuan, and Ma}]{ji2024single}
Ji, J.; Li, G.; Fu, J.; Afghah, F.; Guo, L.; Yuan, X.; and Ma, X. 2024.
\newblock A single-step, sharpness-aware minimization is all you need to
  achieve efficient and accurate sparse training.
\newblock \emph{Advances in Neural Information Processing Systems}, 37:
  44269--44290.

\bibitem[{Jiang et~al.(2020)Jiang, Huang, Liu, and Yang}]{jiang2020beyond}
Jiang, L.; Huang, D.; Liu, M.; and Yang, W. 2020.
\newblock Beyond synthetic noise: Deep learning on controlled noisy labels.
\newblock In \emph{International Conference on Machine Learning}, 4804--4815.
  PMLR.

\bibitem[{Jie and Deng(2023)}]{jie2023fact}
Jie, S.; and Deng, Z.-H. 2023.
\newblock Fact: Factor-tuning for lightweight adaptation on vision transformer.
\newblock In \emph{Proceedings of the AAAI conference on artificial
  intelligence}, volume~37, 1060--1068.

\bibitem[{Kim et~al.(2022)Kim, Li, Hu, and Hospedales}]{kim2022fisher}
Kim, M.; Li, D.; Hu, S.~X.; and Hospedales, T. 2022.
\newblock Fisher sam: Information geometry and sharpness aware minimisation.
\newblock In \emph{International Conference on Machine Learning}, 11148--11161.
  PMLR.

\bibitem[{Kingma and Ba(2015)}]{kingma2015adam}
Kingma, D.~P.; and Ba, J. 2015.
\newblock Adam: A Method for Stochastic Optimization.
\newblock In \emph{ICLR (Poster)}.

\bibitem[{Kolda and Bader(2009)}]{kolda2009tensor}
Kolda, T.~G.; and Bader, B.~W. 2009.
\newblock Tensor decompositions and applications.
\newblock \emph{SIAM review}, 51(3): 455--500.

\bibitem[{Kossaifi et~al.(2019)Kossaifi, Panagakis, Anandkumar, and
  Pantic}]{kossaifi2019tensorly}
Kossaifi, J.; Panagakis, Y.; Anandkumar, A.; and Pantic, M. 2019.
\newblock Tensorly: Tensor learning in python.
\newblock \emph{Journal of Machine Learning Research}, 20(26): 1--6.

\bibitem[{Krizhevsky, Hinton et~al.(2009)}]{krizhevsky2009learning}
Krizhevsky, A.; Hinton, G.; et~al. 2009.
\newblock Learning multiple layers of features from tiny images.

\bibitem[{Kwon et~al.(2021)Kwon, Kim, Park, and Choi}]{kwon2021asam}
Kwon, J.; Kim, J.; Park, H.; and Choi, I.~K. 2021.
\newblock Asam: Adaptive sharpness-aware minimization for scale-invariant
  learning of deep neural networks.
\newblock In \emph{International Conference on Machine Learning}, 5905--5914.
  PMLR.

\bibitem[{Li, Zhang, and He(2024)}]{li2024implicit}
Li, B.; Zhang, L.; and He, N. 2024.
\newblock Implicit regularization of sharpness-aware minimization for
  scale-invariant problems.
\newblock \emph{Advances in Neural Information Processing Systems}, 37:
  44444--44478.

\bibitem[{Li et~al.(2022)Li, Pan, Chen, Ding, Zhao, and Xu}]{li2022heuristic}
Li, N.; Pan, Y.; Chen, Y.; Ding, Z.; Zhao, D.; and Xu, Z. 2022.
\newblock Heuristic rank selection with progressively searching tensor ring
  network.
\newblock \emph{Complex \& Intelligent Systems}, 8(2): 771--785.

\bibitem[{Li et~al.(2025)Li, He, Li, Wang, Shang, and Huang}]{liflat}
Li, T.; He, Z.; Li, Y.; Wang, Y.; Shang, L.; and Huang, X. 2025.
\newblock Flat-LoRA: Low-Rank Adaptation over a Flat Loss Landscape.
\newblock In \emph{Forty-second International Conference on Machine Learning}.

\bibitem[{Li et~al.(2024)Li, Chen, Yang, and Luo}]{li2024efficiency}
Li, Z.; Chen, S.; Yang, J.; and Luo, L. 2024.
\newblock Efficiency calibration of implicit regularization in deep networks
  via self-paced curriculum-driven singular value selection.
\newblock In \emph{Proceedings of the Thirty-Third International Joint
  Conference on Artificial Intelligence, IJCAI-24, International Joint
  Conferences on Artificial Intelligence Organization}, 4515--4523.

\bibitem[{Liu et~al.(2021)Liu, Li, Wei, Zhou, and Zhao}]{liu2021noisy}
Liu, T.; Li, Y.; Wei, S.; Zhou, E.; and Zhao, T. 2021.
\newblock Noisy gradient descent converges to flat minima for nonconvex matrix
  factorization.
\newblock In \emph{International Conference on Artificial Intelligence and
  Statistics}, 1891--1899. PMLR.

\bibitem[{Liu et~al.(2022)Liu, Mai, Cheng, Chen, Hsieh, and
  You}]{liu2022random}
Liu, Y.; Mai, S.; Cheng, M.; Chen, X.; Hsieh, C.-J.; and You, Y. 2022.
\newblock Random sharpness-aware minimization.
\newblock \emph{Advances in Neural Information Processing Systems}, 35:
  24543--24556.

\bibitem[{Liu et~al.(2019)Liu, Ott, Goyal, Du, Joshi, Chen, Levy, Lewis,
  Zettlemoyer, and Stoyanov}]{liu2019roberta}
Liu, Y.; Ott, M.; Goyal, N.; Du, J.; Joshi, M.; Chen, D.; Levy, O.; Lewis, M.;
  Zettlemoyer, L.; and Stoyanov, V. 2019.
\newblock Roberta: A robustly optimized bert pretraining approach.
\newblock \emph{arXiv preprint arXiv:1907.11692}.

\bibitem[{Loeschcke et~al.(2024)Loeschcke, Wang, Leth-Espensen, Belongie,
  Kastoryano, and Benaim}]{loeschcke2024coarse}
Loeschcke, S.~B.; Wang, D.; Leth-Espensen, C.~M.; Belongie, S.; Kastoryano, M.;
  and Benaim, S. 2024.
\newblock Coarse-To-Fine Tensor Trains for Compact Visual Representations.
\newblock In \emph{International Conference on Machine Learning}, 32612--32642.
  PMLR.

\bibitem[{Loshchilov and Hutter(2019)}]{loshchilov2018decoupled}
Loshchilov, I.; and Hutter, F. 2019.
\newblock Decoupled Weight Decay Regularization.
\newblock In \emph{International Conference on Learning Representations}.

\bibitem[{Malladi et~al.(2023)Malladi, Gao, Nichani, Damian, Lee, Chen, and
  Arora}]{malladi2023fine}
Malladi, S.; Gao, T.; Nichani, E.; Damian, A.; Lee, J.~D.; Chen, D.; and Arora,
  S. 2023.
\newblock Fine-tuning language models with just forward passes.
\newblock \emph{Advances in Neural Information Processing Systems}, 36:
  53038--53075.

\bibitem[{Memmel et~al.(2024)Memmel, Menzen, Schuurmans, Wesel, and
  Batselier}]{memmel2024position}
Memmel, E.; Menzen, C.; Schuurmans, J.; Wesel, F.; and Batselier, K. 2024.
\newblock Position: Tensor Networks are a Valuable Asset for Green AI.
\newblock In \emph{International Conference on Machine Learning}, 35340--35353.
  PMLR.

\bibitem[{Mueller et~al.(2023)Mueller, Vlaar, Rolnick, and
  Hein}]{mueller2023normalization}
Mueller, M.; Vlaar, T.; Rolnick, D.; and Hein, M. 2023.
\newblock Normalization layers are all that sharpness-aware minimization needs.
\newblock \emph{Advances in Neural Information Processing Systems}, 36:
  69228--69252.

\bibitem[{Neyshabur et~al.(2017)Neyshabur, Bhojanapalli, McAllester, and
  Srebro}]{neyshabur2017exploring}
Neyshabur, B.; Bhojanapalli, S.; McAllester, D.; and Srebro, N. 2017.
\newblock Exploring generalization in deep learning.
\newblock \emph{Advances in Neural Information Processing Systems}, 30.

\bibitem[{Novikov et~al.(2015)Novikov, Podoprikhin, Osokin, and
  Vetrov}]{novikov2015tensorizing}
Novikov, A.; Podoprikhin, D.; Osokin, A.; and Vetrov, D.~P. 2015.
\newblock Tensorizing neural networks.
\newblock \emph{Advances in Neural Information Processing Systems}, 28.

\bibitem[{Pan et~al.(2022)Pan, Su, Liu, Jingquan, Li, and Xu}]{pan2022unified}
Pan, Y.; Su, Z.; Liu, A.; Jingquan, W.; Li, N.; and Xu, Z. 2022.
\newblock A unified weight initialization paradigm for tensorial convolutional
  neural networks.
\newblock In \emph{International conference on machine learning}, 17238--17257.
  PMLR.

\bibitem[{Paszke et~al.(2019)Paszke, Gross, Massa, Lerer, Bradbury, Chanan,
  Killeen, Lin, Gimelshein, Antiga et~al.}]{paszke2019pytorch}
Paszke, A.; Gross, S.; Massa, F.; Lerer, A.; Bradbury, J.; Chanan, G.; Killeen,
  T.; Lin, Z.; Gimelshein, N.; Antiga, L.; et~al. 2019.
\newblock Pytorch: An imperative style, high-performance deep learning library.
\newblock \emph{Advances in Neural Information Processing Systems}, 32.

\bibitem[{Phan et~al.(2020)Phan, Sobolev, Sozykin, Ermilov, Gusak,
  Tichavsk{\`y}, Glukhov, Oseledets, and Cichocki}]{phan2020stable}
Phan, A.-H.; Sobolev, K.; Sozykin, K.; Ermilov, D.; Gusak, J.; Tichavsk{\`y},
  P.; Glukhov, V.; Oseledets, I.; and Cichocki, A. 2020.
\newblock Stable low-rank tensor decomposition for compression of convolutional
  neural network.
\newblock In \emph{European Conference on Computer Vision}, 522--539. Springer.

\bibitem[{Rajpurkar et~al.(2016)Rajpurkar, Zhang, Lopyrev, and
  Liang}]{rajpurkar2016squad}
Rajpurkar, P.; Zhang, J.; Lopyrev, K.; and Liang, P. 2016.
\newblock SQuAD: 100,000+ Questions for Machine Comprehension of Text.
\newblock In \emph{Proceedings of the 2016 Conference on Empirical Methods in
  Natural Language Processing}, 2383--2392.

\bibitem[{Razin, Maman, and Cohen(2022)}]{razin2022implicit}
Razin, N.; Maman, A.; and Cohen, N. 2022.
\newblock Implicit regularization in hierarchical tensor factorization and deep
  convolutional neural networks.
\newblock In \emph{International Conference on Machine Learning}, 18422--18462.
  PMLR.

\bibitem[{Si et~al.(2025)Si, Wang, Yang, Xu, Li, Dai, Qiao, Yang, and
  Shen}]{si2025maintaining}
Si, C.; Wang, X.; Yang, X.; Xu, Z.; Li, Q.; Dai, J.; Qiao, Y.; Yang, X.; and
  Shen, W. 2025.
\newblock Maintaining Structural Integrity in Parameter Spaces for Parameter
  Efficient Fine-tuning.
\newblock In \emph{The Thirteenth International Conference on Learning
  Representations}.

\bibitem[{Smith and Topin(2019)}]{smith2019super}
Smith, L.~N.; and Topin, N. 2019.
\newblock Super-convergence: Very fast training of neural networks using large
  learning rates.
\newblock In \emph{Artificial intelligence and machine learning for
  multi-domain operations applications}, volume 11006, 369--386. SPIE.

\bibitem[{Srivastava et~al.(2014)Srivastava, Hinton, Krizhevsky, Sutskever, and
  Salakhutdinov}]{srivastava2014dropout}
Srivastava, N.; Hinton, G.; Krizhevsky, A.; Sutskever, I.; and Salakhutdinov,
  R. 2014.
\newblock Dropout: a simple way to prevent neural networks from overfitting.
\newblock \emph{Journal of Machine Learning Research}, 15(1): 1929--1958.

\bibitem[{Van~Rooyen, Menon, and Williamson(2015)}]{van2015learning}
Van~Rooyen, B.; Menon, A.; and Williamson, R.~C. 2015.
\newblock Learning with symmetric label noise: The importance of being
  unhinged.
\newblock \emph{Advances in Neural Information Processing Systems}, 28.

\bibitem[{Veeramacheneni et~al.(2025)Veeramacheneni, Wolter, Kuehne, and
  Gall}]{veeramacheneni2025canonical}
Veeramacheneni, L.; Wolter, M.; Kuehne, H.; and Gall, J. 2025.
\newblock Canonical Rank Adaptation: An Efficient Fine-Tuning Strategy for
  Vision Transformers.
\newblock In \emph{Forty-second International Conference on Machine Learning}.

\bibitem[{Wang et~al.(2019)Wang, Pruksachatkun, Nangia, Singh, Michael, Hill,
  Levy, and Bowman}]{wang2019superglue}
Wang, A.; Pruksachatkun, Y.; Nangia, N.; Singh, A.; Michael, J.; Hill, F.;
  Levy, O.; and Bowman, S. 2019.
\newblock Superglue: A stickier benchmark for general-purpose language
  understanding systems.
\newblock \emph{Advances in Neural Information Processing Systems}, 32.

\bibitem[{Wang et~al.(2024)Wang, Qiu, Bai, Jin, Zhou, and
  Zhao}]{wang2024generalized}
Wang, A.; Qiu, Y.; Bai, M.; Jin, Z.; Zhou, G.; and Zhao, Q. 2024.
\newblock Generalized tensor decomposition for understanding multi-output
  regression under combinatorial shifts.
\newblock \emph{Advances in Neural Information Processing Systems}, 37:
  47559--47635.

\bibitem[{Wang et~al.(2018)Wang, Sun, Eriksson, Wang, and
  Aggarwal}]{wang2018wide}
Wang, W.; Sun, Y.; Eriksson, B.; Wang, W.; and Aggarwal, V. 2018.
\newblock Wide compression: Tensor ring nets.
\newblock In \emph{Proceedings of the IEEE Conference on Computer Vision and
  Pattern Recognition}, 9329--9338.

\bibitem[{Wei, Kakade, and Ma(2020)}]{wei2020implicit}
Wei, C.; Kakade, S.; and Ma, T. 2020.
\newblock The implicit and explicit regularization effects of dropout.
\newblock In \emph{International Conference on Machine Learning}, 10181--10192.
  PMLR.

\bibitem[{Wen, Ma, and Li(2022)}]{wen2022sharpness}
Wen, K.; Ma, T.; and Li, Z. 2022.
\newblock How sharpness-aware minimization minimizes sharpness?
\newblock In \emph{The Eleventh International Conference on Learning
  Representations}.

\bibitem[{Xie, Pethick, and Cevher(2024)}]{xie2024sampa}
Xie, W.; Pethick, T.; and Cevher, V. 2024.
\newblock Sampa: Sharpness-aware minimization parallelized.
\newblock \emph{Advances in Neural Information Processing Systems}, 37:
  51333--51357.

\bibitem[{Yang et~al.(2024)Yang, Zhou, Wong, and Zhang}]{yang2024loretta}
Yang, Y.; Zhou, J.; Wong, N.; and Zhang, Z. 2024.
\newblock LoRETTA: Low-Rank Economic Tensor-Train Adaptation for
  Ultra-Low-Parameter Fine-Tuning of Large Language Models.
\newblock In \emph{Proceedings of the 2024 Conference of the North American
  Chapter of the Association for Computational Linguistics: Human Language
  Technologies (Volume 1: Long Papers)}, 3161--3176.

\bibitem[{Yaras et~al.(2024)Yaras, Wang, Balzano, and
  Qu}]{yaras2024compressible}
Yaras, C.; Wang, P.; Balzano, L.; and Qu, Q. 2024.
\newblock Compressible Dynamics in Deep Overparameterized Low-Rank Learning \&
  Adaptation.
\newblock In \emph{International Conference on Machine Learning}, 56946--56965.
  PMLR.

\bibitem[{Yin et~al.(2021)Yin, Sui, Liao, and Yuan}]{yin2021towards}
Yin, M.; Sui, Y.; Liao, S.; and Yuan, B. 2021.
\newblock Towards efficient tensor decomposition-based dnn model compression
  with optimization framework.
\newblock In \emph{Proceedings of the IEEE/CVF Conference on Computer Vision
  and Pattern Recognition}, 10674--10683.

\bibitem[{Zhang et~al.(2021)Zhang, Bengio, Hardt, Recht, and
  Vinyals}]{zhang2021understanding}
Zhang, C.; Bengio, S.; Hardt, M.; Recht, B.; and Vinyals, O. 2021.
\newblock Understanding deep learning (still) requires rethinking
  generalization.
\newblock \emph{Communications of the ACM}, 64(3): 107--115.

\bibitem[{Zhang et~al.(2022)Zhang, Roller, Goyal, Artetxe, Chen, Chen, Dewan,
  Diab, Li, Lin et~al.}]{zhang2022opt}
Zhang, S.; Roller, S.; Goyal, N.; Artetxe, M.; Chen, M.; Chen, S.; Dewan, C.;
  Diab, M.; Li, X.; Lin, X.~V.; et~al. 2022.
\newblock Opt: Open pre-trained transformer language models.
\newblock \emph{arXiv preprint arXiv:2205.01068}.

\bibitem[{Zhao et~al.(2016)Zhao, Zhou, Xie, Zhang, and
  Cichocki}]{zhao2016tensor}
Zhao, Q.; Zhou, G.; Xie, S.; Zhang, L.; and Cichocki, A. 2016.
\newblock Tensor ring decomposition.
\newblock \emph{arXiv preprint arXiv:1606.05535}.

\bibitem[{Zhuang et~al.(2022)Zhuang, Gong, Yuan, Cui, Adam, Dvornek, sekhar
  tatikonda, s~Duncan, and Liu}]{zhuang2022surrogate}
Zhuang, J.; Gong, B.; Yuan, L.; Cui, Y.; Adam, H.; Dvornek, N.~C.; sekhar
  tatikonda; s~Duncan, J.; and Liu, T. 2022.
\newblock Surrogate Gap Minimization Improves Sharpness-Aware Training.
\newblock In \emph{International Conference on Learning Representations}.

\end{thebibliography}

\appendix

\section{Missing proof and results}

\subsection{Lemma~\ref{lemma:gradient-derivative} and proof}

\begin{lemma}
\label{lemma:gradient-derivative}
For any tensor $\mathcal V$ that has the same shape as $\mathcal G_m$,
$$
\left\langle \Phi (\ldots, \mathcal G_{m-1}, \mathcal V, \mathcal G_{m+1},\ldots),
\nabla_{\mathcal T} f\right\rangle_F = \left\langle \mathcal V, \nabla_{\mathcal G_m} f\right\rangle_F,
$$
where $\mathcal T = \Phi(\mathcal G_1, \ldots, \mathcal G_m, \ldots, \mathcal G_K)$.
\end{lemma}
\begin{proof}
Let $ F(\mathcal G_1, \ldots, \mathcal G_K) := f(\Phi(\mathcal G_1, \ldots, \mathcal G_m,\ldots, \mathcal G_K)) $, and fix all $ G_j $ except $ G_m $. 
The directional derivative of $F$ at $\mathcal G_m$ in the direction of $\mathcal V$ is given by:
\begin{align*}
D_{\mathcal G_m} F(\mathcal G_m; \mathcal V) & = \frac{d}{d\epsilon} F(\ldots,\mathcal G_m + \epsilon \mathcal V, \ldots) \bigg|_{\epsilon=0}\\
&= \left\langle \nabla_{\mathcal T} f, \frac{d}{d\epsilon} \Phi(\ldots, \mathcal G_m + \epsilon \mathcal V,\ldots) \bigg|_{\epsilon=0} \right\rangle_F\\
&= \left\langle \nabla_{\mathcal T} f, \Phi(\ldots, \mathcal V, \ldots) \right\rangle_F.
\end{align*}
By the definition of gradient and chain rule,
we have:
$$D_{\mathcal G_m} F(\mathcal G_m; \mathcal V) = \left\langle \nabla_{\mathcal G_m} f, \mathcal V \right\rangle_F.$$
Then, we arrive at the desired result:
\begin{equation}
    \label{eq:Gm-T-inner}
\left\langle \Phi (\ldots, \mathcal V,\ldots) , \nabla_{\mathcal T} f\right\rangle_F  = \left\langle \nabla_{\mathcal G_m} f, \mathcal V \right\rangle_F.
\end{equation}
This completes the proof.
\end{proof}

\subsection{Proof of Theorem~\ref{thm:equivalent-norm}}

\begin{proof}
    For any core $\mathcal G_k^{(t)}, \forall k\in [K]$, the dynamics of its squared Frobenius norm is given by:
    $$
    \frac{d}{dt}\|\mathcal G_k^{(t)}\|^2_F = 2 \left\langle \mathcal G_k^{(t)}, \frac{d}{dt} \mathcal G_k^{(t)} \right\rangle_F = -2 \left\langle \mathcal G_k^{(t)}, \nabla_{\mathcal G_k^{(t)}} f \right\rangle_F.
    $$
    By the lemma \eqref{lemma:gradient-derivative}, we have:
    \begin{align*}
    \left\langle \mathcal G_k^{(t)}, \nabla_{\mathcal G_k^{(t)}} f \right\rangle_F &= \left\langle \Phi(\ldots, \mathcal G_{k-1}^{(t)}, \mathcal G_k^{(t)}, \mathcal G_{k+1}^{(t)}, \ldots), \nabla_{\mathcal T} f \right\rangle_F\\
    &= \left\langle \mathcal T, \nabla_{\mathcal T} f \right\rangle_F.
    \end{align*}
    Therefore, all cores have the same dynamics as follows:
    $$
    \frac{d}{dt}\|\mathcal G_1^{(t)}\|^2_F = \cdots = \frac{d}{dt}\|\mathcal G_K^{(t)}\|^2_F
 = -2 \left\langle \mathcal T, \nabla_{\mathcal T} f \right\rangle_F.
    $$ This completes the proof.
\end{proof}

\subsection{Proof of Theorem~\ref{thm:sam-dynamics}}
\begin{proof}
We have denoted gradients at the perturbed point and unperturbed point as $\tilde{g}_k^{(t)}$ and $g_k^{(t)}$ respectively
 in \eqref{eq:sam-update-step}. 
Denote ${\mathcal T}^{(t)} = \Phi({\mathcal G}_1^{(t)}, \ldots, {\mathcal G}_K^{(t)})$, $\tilde{\mathcal T}^{(t)} = \Phi(\tilde{\mathcal G}_1^{(t)}, \ldots, \tilde{\mathcal G}_K^{(t)})$,
and $\tilde{f} = f(\tilde{\mathcal T}^{(t)})$.
Note that perturbed point $\tilde{\mathcal G}_k^{(t)}$ satisfies Lemma \ref{lemma:gradient-derivative}.
By letting $\mathcal V = \tilde{\mathcal G}_k^{(t)}$, we have:
$$
\left\langle \tilde{\mathcal T}^{(t)}, \nabla_{\mathcal T} \tilde{f} \right\rangle_F = \left\langle \tilde{\mathcal G}_k^{(t)}, \tilde g_{k}^{(t)} \right\rangle_F.
$$
Next, we expand the dynamics of $\forall i,j \in [K]$ with $i\neq j$ as follows:
\begin{align*}
& \frac{d}{dt}\left(\|\mathcal G_i^{(t)}\|^2_F - \|\mathcal G_j^{(t)}\|^2_F\right) \\
&= -2\left( \left\langle \mathcal G_i^{(t)}, \tilde g_i^{(t)} \right\rangle_F - \left\langle \mathcal G_j^{(t)}, \tilde g_j^{(t)} \right\rangle_F\right)\\
&= -2\bigg( \bigg[\underbrace{\left\langle \tilde{\mathcal G}_i^{(t)}, \tilde g_i^{(t)} \right\rangle_F}
_{=\left\langle \tilde{\mathcal T}^{(t)}, \nabla_{\mathcal T} \tilde{f} \right\rangle_F} - \rho u^{(t)} \left\langle g_i^{(t)}, \tilde g_i^{(t)} \right\rangle_F\bigg]\\
& \quad - \bigg[\underbrace{\left\langle \tilde{\mathcal G}_j^{(t)}, \tilde g_j^{(t)} \right\rangle_F}_{=\left\langle \tilde{\mathcal T}^{(t)}, \nabla_{\mathcal T} \tilde{f} \right\rangle_F} - \rho u^{(t)} \left\langle g_j^{(t)}, \tilde g_j^{(t)} \right\rangle_F \bigg]\bigg)\\
& = 2\rho u^{(t)} \left(\left\langle g_i^{(t)}, \tilde g_i^{(t)} \right\rangle_F - \left\langle g_j^{(t)}, \tilde g_j^{(t)} \right\rangle_F\right)\\
& = 2\rho u^{(t)} \left(\|g_i^{(t)}\|^2_F - \|g_j^{(t)}\|^2_F\right) \\
& \quad + \underbrace{2\rho u^{(t)} \left(\left\langle g_i^{(t)}, \tilde g_i^{(t)} - g_i^{(t)} \right\rangle_F - \left\langle g_j^{(t)}, \tilde g_j^{(t)} - g_j^{(t)} \right\rangle_F\right)}_{=: R^{(t)}_{ij}}.
\end{align*}
To bound the term \( R^{(t)}_{ij} \), we can bound the term for $g_i^{(t)}$ and $g_j^{(t)}$ separately.
First, we have:
\begin{align*}
2\rho u^{(t)}\left\langle g_i^{(t)}, \tilde g_i^{(t)} - g_i^{(t)} \right\rangle_F 
& \leq 2\rho \underbrace{u^{(t)} \|g_i^{(t)}\|_F}_{\leq 1} \|\tilde g_i^{(t)} - g_i^{(t)}\|_F\\
& \leq 2\rho \|\tilde g_i^{(t)} - g_i^{(t)}\|_F.
\end{align*}
That is, we need to bound the term \( \|\tilde g_i^{(t)} - g_i^{(t)}\|_F \).
With assumption \eqref{assumption:smoothness}, we have:
$$
\|\tilde g_i^{(t)} - g_i^{(t)}\|_F \leq L \|\tilde{\mathcal T}^{(t)} - \mathcal T^{(t)}\|_F.
$$
Define $\Delta_k = \tilde{\mathcal G}_k^{(t)} - \mathcal G_k^{(t)}$. 
We have that $\|\Delta_k\|_F = \rho u^{(t)}\|g_k^{(t)}\|_F  = O(\rho)$.
We can expand
\begin{align*}
    \tilde{\mathcal T}^{(t)} - \mathcal T^{(t)}
    &= \underbrace{\sum_{k=1}^K\Phi(\mathcal G_1^{(t)}, \ldots, \Delta_k, \ldots, \mathcal G_K^{(t)})}_{=: A^{(1)}}\\
    & \quad + \underbrace{\sum_{1\leq i<j \leq K} \Phi(\ldots,\Delta_i,\ldots,\Delta_j,\ldots)}_{=: A^{(2)}}\\
    & \quad + A^{(>2)},
\end{align*}
where $A^{(>2)}$ collects all terms involving at least three perturbations.
The norm of $A^{(1)}$ can be bounded as follows:
\begin{align*}
\|A^{(1)}\|_F &\leq \sum_{k=1}^K \|\Phi(\mathcal G_1^{(t)}, \ldots, \Delta_k, \ldots, \mathcal G_K^{(t)})\|_F \\
&\leq \sum_{k=1}^K \left(\prod_{j\neq k}\|\mathcal G_j\|_F\right)\cdot \|\Delta_k\|_F\\
&= O(\rho).
\end{align*}
For the term \( A^{(2)} \), we have:
\begin{align*}
\|A^{(2)}\|_F &\leq \sum_{1\leq i<j \leq K} \|\Phi(\ldots,\Delta_i,\ldots,\Delta_j,\ldots)\|_F \\
&\leq \sum_{1\leq i<j \leq K} \left(\prod_{k\neq i,j}\|\mathcal G_k\|_F\right) \cdot \|\Delta_i\|_F \cdot \|\Delta_j\|_F\\
&= O(\rho^2).
\end{align*}
Similarly, higher-order terms are all $O(\rho^k)$ with $k\geq 2$, and we can bound:
\begin{align*}
    \|\tilde{\mathcal T}^{(t)} - \mathcal T^{(t)}\|_F
    &\leq \|A^{(1)}\|_F + \|A^{(2)}\|_F + \|A^{(>2)}\|_F\\
    &\leq O(\rho) + O(\rho^2) + O(\rho^3) + \cdots\\
    &= O(\rho).
\end{align*}
Thus, we have:
$$\|\tilde g_i^{(t)} - g_i^{(t)}\|_F \leq L \|\tilde{\mathcal T}^{(t)} - \mathcal T^{(t)}\|_F = O(\rho L).$$
Combining the bounds, we have:
\begin{align*}
R^{(t)}_{ij} &\leq 2\rho \|\tilde g_i^{(t)} - g_i^{(t)}\|_F + 2\rho \|\tilde g_j^{(t)} - g_j^{(t)}\|_F \\
&= O(\rho^2 L).
\end{align*}
The proof is complete.
\end{proof}

\subsection{Proof of Corollary~\ref{cor:sgd-norm-deviation}}

\begin{proof}
    By Theorem~\ref{thm:equivalent-norm}, we have that $\forall k,j\in [K]$,
    $$
    \frac{d}{dt} \left(\|\mathcal G_k^{(t)}\|^2_F - \|\mathcal G_j^{(t)}\|^2_F\right)=0.
    $$
    Therefore, the Norm Deviation $Q$ is constant:
    \begin{align*}
    \frac{dQ}{dt} &= \frac{d}{dt}\frac{1}{2K}\sum_{i,j=1}^K \left(\|\mathcal G_i\|_F^2 - \|\mathcal G_j\|_F^2\right)^2 = 0.
    \end{align*}
    This completes the proof.
\end{proof}

\begin{lemma}
\label{lemma:core-mean}
For the cores $\{\mathcal G_k\}_{k=1}^K$, with infinitesimal stepsize $\eta \to 0$ and
update step \eqref{eq:sam-update-step}, 
\begin{align*}
&\frac{d}{dt}\left(\|\mathcal G_k^{(t)}\|
^2_F - \frac 1K \sum_{i=1}^K \|\mathcal G_i^{(t)}\|
^2_F\right) \\
& = 2\rho u^{(t)} \left(\|g_k^{(t)}\|^2_F - \frac 1K \sum_{i=1}^K \|g_i^{(t)}\|^2_F\right) + O(\rho^2 L).
\end{align*}
\end{lemma}
\begin{proof}
    By Theorem~\ref{thm:sam-dynamics}, we have:
    \begin{align*}
    &\frac{d}{dt}\left(\|\mathcal G_k^{(t)}\|^2_F - \|\mathcal G_j^{(t)}\|^2_F\right) \\
    &= 2\rho u^{(t)} \left(\|g_k^{(t)}\|^2_F - \|g_j^{(t)}\|^2_F\right) + O(\rho^2 L).
    \end{align*}
    Summing over $j$ and dividing by $K$ gives:
    \begin{align*}
    &\frac{d}{dt}\left(\|\mathcal G_k^{(t)}\|^2_F - \frac 1K \sum_{i=1}^K \|\mathcal G_i^{(t)}\|^2_F\right)\\
    & = 2\rho u^{(t)} \left(\|g_k^{(t)}\|^2_F - \frac 1K \sum_{i=1}^K \|g_i^{(t)}\|^2_F\right) + O(\rho^2 L).
    \end{align*}
    This completes the proof.
\end{proof}

\subsection{Proof of Theorem~\ref{thm:sam-norm-deviation}}
\begin{proof}
For simplicity, let $\bar{s} := \frac 1K \sum_{i=1}^K \|\mathcal G_k^{(t)}\|^2_F$ and 
$\bar{\gamma} := \frac 1K \sum_{i=1}^K \|g_k^{(t)}\|^2_F$. The dynamics of $Q$ 
can be expressed as:
\begin{align*}
\frac{dQ}{dt} &= \frac{d}{dt}\sum_{k=1}^K \left(\|\mathcal G_k^{(t)}\|_F^2 - \bar{s}\right)^2 \\
&= \sum_{k=1}^K 2\left(\|\mathcal G_k^{(t)}\|_F^2 - \bar{s}\right) \frac{d}{dt}\left(\|\mathcal G_k^{(t)}\|_F^2 - \bar{s}\right)\\
&= 4\rho u^{(t)} \sum_{k=1}^K \left(\|\mathcal G_k^{(t)}\|_F^2 - \bar{s}\right)  \left(\|g_k^{(t)}\|_F^2 - \bar{\gamma}\right) + O(\rho^2 L)\\
&= 4\rho u^{(t)} K\cdot \mathrm{Cov}\left(\|\mathcal G^{(t)}_k\|_F^2,\|g_k^{(t)}\|_F^2\right) + O(\rho^2 L),
\end{align*}
where the third equality follows from Lemma \ref{lemma:core-mean}.
\end{proof}

\subsection{Proposition~\ref{prop:pairwise-norm-shrinkage} and proof}

\begin{proposition}[Local Pairwise Norm Shrinkage in General Scale-invariant Models]
    \label{prop:pairwise-norm-shrinkage}
    We suppose that $\|g_i\|^2_F>0$, $\|g_j\|^2_F>0$, and small enough $\rho>0$.
    Under SAM gradient flow, the pairwise norm gap $\mathcal{B}_{ij}:=
\big|\|\mathcal G_i\|^2_F - \|\mathcal G_j\|^2_F\big|$ 
satisfies:
\[
\frac{d}{dt} \mathcal{B}_{ij} < 0
\quad \text{whenever } \mathcal{B}_{ij} > \bar{\mathcal{B}}_{ij}^{(\rho)} := \left( \bar{\alpha}^2 + \delta \right) - \frac{1}{\bar{\alpha}^2 + \delta},
\]
where
\[
\bar{\alpha}^2 := \sqrt{\frac{C_i}{C_j}}, \ \delta = \mathcal{O}(\rho),
\]
for some constant $C_i, C_j$.
\end{proposition}

\begin{proof}
Without loss of generality, suppose that $\|\mathcal G_i\|^2_F > \|\mathcal G_j\|^2_F$
and $\mathcal{B}_{ij}=
\|\mathcal G_i\|^2_F - \|\mathcal G_j\|^2_F$ .
By Theorem~\ref{thm:sam-dynamics}, the evolution of the pairwise norm difference is given by:
\begin{align*}
\frac{d}{dt} &\left( \|\mathcal G_i\|_F^2 - \|\mathcal G_j\|_F^2 \right) \\ &= 2\rho u^{(t)} \left( \|g_i\|_F^2 - \|g_j\|_F^2 \right) + O(\rho^2L).
\end{align*}

Under the re-parameterization $\mathcal G_i = \alpha \bar{\mathcal G}_i$, $\mathcal G_j = \bar{\mathcal G}_j / \alpha$ with $\|\bar{\mathcal G}_i\|_F = \|\bar{\mathcal G}_j\|_F = G$, we have:
\[
\|\mathcal G_i\|_F^2 = \alpha^2G^2, \quad \|\mathcal G_j\|_F^2 = \frac{G^2}{\alpha^2} \Rightarrow \mathcal{B}_{ij} = (\alpha^2 - \frac{1}{\alpha^2})\cdot G^2.
\]

Since $\Phi(\cdot)$ is multilinear and the loss $f$ is fixed under this re-parameterization, the gradients scale accordingly:
\[
\|g_i\|_F^2 = \frac{C_i}{\alpha^2}, \quad \|g_j\|_F^2 = \alpha^2 C_j,
\]
for constants $C_i = \|\nabla_{\bar{\mathcal G}_i} f\|_F^2$ and $C_j = \|\nabla_{\bar{\mathcal G}_j} f\|_F^2$.

We now show that the leading-order term in the SAM flow dominates the higher-order correction. Define
\[
\epsilon := \frac{C_i}{\alpha^2} - \alpha^2 C_j.
\]
Then:
\[
\frac{d\mathcal{B}_{ij}}{dt} = 2\rho u^{(t)} \epsilon + \mathcal{O}(\rho^2 L).
\]
If \( \epsilon < -c \rho L \) for some constant \( c > 0 \), the first-order term dominates and we obtain:
\[
\frac{d\mathcal{B}_{ij}}{dt} < 0.
\]
This occurs whenever \( \alpha^2 > \bar{\alpha}^2 := \sqrt{C_i / C_j} + \delta \) for some small margin \( \delta = \mathcal{O}(\rho) \). In that case,
\[
\mathcal{B}_{ij} > ( \bar{\alpha}^2 - \frac{1}{\bar{\alpha}^2} )\cdot G^2
\Rightarrow \frac{d\mathcal{B}_{ij}}{dt} < 0.
\]
which completes the proof.
\end{proof}

\subsection{Results for Norm Deviations of multi-layer models}
We extend the results in the main part of the paper following~\cite{li2024implicit} and let $l\in [D]$ be the layer index. Denote the multilinear reconstruction function of $l$-th layer be $\Phi_l$ and the corresponding tensor core set be $\{\mathcal G_{1,l}^{(t)}, \ldots, \mathcal G_{K_l,l}^{(t)}\}$. Then the update of SAM for layer $l$ can be written as:

\begin{equation}
\label{eq:multi-layer-sam-update-step}
\begin{aligned}
    &g_{k,l}^{(t)} = \nabla_{\mathcal G_{k,l}} f(\{\Phi_l(\mathcal G_{1,l}^{(t)}, \ldots, \mathcal G_{K_l,l}^{(t)})\}_l), \\
    &\tilde{\mathcal G}_{k,l}^{(t)} = \mathcal G_{k,l}^{(t)} + \rho u_D^{(t)} g_{k,l}^{(t)},\\
    &\tilde{g}_{k,l}^{(t)} = \nabla_{\mathcal G_{k,l}} f(\{\Phi_l(\tilde{\mathcal G}_{1,l}^{(t)}, \ldots, \tilde{\mathcal G}_{K_l,l}^{(t)})\}_l),\\
    &\mathcal G_{k,l}^{(t+1)} = \mathcal G_{k,l}^{(t)} - \eta \tilde{g}_{k,l}^{(t)},
\end{aligned}
\end{equation}
where $u_D^{(t)}= (\sum_{l=1}^D \sum_{k=1}^{K_l}\|\mathcal g_{k,l}^{(t)}\|^2_F)^{-1/2}$ is the normalization factor.
We use the following assumption:
\begin{assumption}[Layer-wise Smoothness]\label{assumption:mult-layer-smoothness}
There exists $\hat L > 0$ such that for any real tensors $\mathcal X_l, \mathcal Y_l$ having the same shape as $\Phi_l(\mathcal G_{1,l}^{(t)}, \ldots, \mathcal G_{K_l,l}^{(t)})$ for all $l$, it holds that
$$\|\nabla_l f(\mathcal{X}_l) - \nabla_l f(\mathcal{Y}_l)\|_F \leq \hat L \|\mathcal{X}_l - \mathcal{Y}_l\|_F,$$
where $\nabla_l f$ is the gradient on $\mathcal X_l$.
\end{assumption}
Theorem~\ref{thm:sam-dynamics} can be extended to
\begin{theorem}[Layer-wise Pairwise Norm Dynamics under SAM]
\label{thm:multi-layer-sam-dynamics}
Applying the update steps \eqref{eq:multi-layer-sam-update-step} with infinitesimal stepsize $\eta \to 0$, the gradient flow
of SAM satisfies that $\forall i,j \in [K]$ with $i\neq j$:
\begin{align*}
\frac{d}{dt}\left(\|\mathcal G_i^{(t)}\|_F^2 - \|\mathcal G_j^{(t)}\|_F^2\right) 
&= 2\rho u_D^{(t)} \left(\|g_i^{(t)}\|_F^2 - \|g_j^{(t)}\|_F^2\right) \\
&\quad + O(\rho^2 \hat L).
\end{align*}
\end{theorem}

\begin{proof}
Using similar arguments in the proof of Theorem~\ref{thm:sam-dynamics}, $\forall i,j \in [K_l], l\in[D]$ with $i\neq j$ we have
\begin{align*}
& \frac{d}{dt}\left(\|\mathcal G_{i,l}^{(t)}\|^2_F - \|\mathcal G_{j,l}^{(t)}\|^2_F\right) \\
& = 2\rho u_D^{(t)} \left(\|g_{i,l}^{(t)}\|^2_F - \|g_{j,l}^{(t)}\|^2_F\right) \\
& \quad + \underbrace{2\rho u_D^{(t)} \left(\left\langle g_{i,l}^{(t)}, \tilde g_{i,l}^{(t)} - g_{i,l}^{(t)} \right\rangle_F - \left\langle g_{j,l}^{(t)}, \tilde g_{j,l}^{(t)} - g_{j,l}^{(t)} \right\rangle_F\right)}_{=: R^{(t)}_{ij,l}}.
\end{align*}

To bound the term $R^{(t)}_{ij,l}$, we bound the following term:
\begin{align*}
2\rho u_D^{(t)}\left\langle g_{i,l}^{(t)}, \tilde g_{i,l}^{(t)} - g_{i,l}^{(t)} \right\rangle_F 
& \leq 2\rho \|\tilde g_{i,l}^{(t)} - g_{i,l}^{(t)}\|_F\\
&\leq 2\rho \hat L\|\tilde{\mathcal T}_l^{(t)} - \mathcal T_l^{(t)}\|,
\end{align*}
where the second $\leq$ uses Assumption~\ref{assumption:mult-layer-smoothness} and we denote ${\mathcal T}_l^{(t)} = \Phi({\mathcal G}_{1,l}^{(t)}, \ldots, {\mathcal G}_{K_l,l}^{(t)})$, $\tilde{\mathcal T}_l^{(t)} = \Phi(\tilde{\mathcal G}_{1,l}^{(t)}, \ldots, \tilde{\mathcal G}_{K_l,l}^{(t)})$.
Define $\Delta_{k,l} = \tilde{\mathcal G}_{k,l}^{(t)} - \mathcal G_{k,l}^{(t)}$. 
We have that $\|\Delta_{k,l}\|_F = \rho u_D^{(t)}\|g_{k,l}^{(t)}\|_F  = O(\rho)$.
We can expand
\begin{align*}
    \tilde{\mathcal T_l}^{(t)} - \mathcal T_l^{(t)}
    &= \underbrace{\sum_{k=1}^{K_l}\Phi(\mathcal G_{1,l}^{(t)}, \ldots, \Delta_{k,l}, \ldots, \mathcal G_{K_l,l}^{(t)})}_{=: A_l^{(1)}}\\
    & \quad + \underbrace{\sum_{1\leq i<j \leq K_l} \Phi(\ldots,\Delta_{i,l},\ldots,\Delta_{j,l},\ldots)}_{=: A_l^{(2)}}\\
    & \quad + A_l^{(>2)},
\end{align*}
where $A_l^{(>2)}$ collects all terms involving at least three perturbations.
The norm of $A_l^{(1)}$ can be bounded as follows:
\begin{align*}
\|A_l^{(1)}\|_F &\leq \sum_{k=1}^{K_l} \|\Phi(\mathcal G_{1,l}^{(t)}, \ldots, \Delta_{k,l}, \ldots, \mathcal G_{K_l,l}^{(t)})\|_F \\
&\leq \sum_{k=1}^{K_l} \left(\prod_{j\neq k}\|\mathcal G_{j,l}\|_F\right)\cdot \|\Delta_{k,l}\|_F\\
&= O(\rho).
\end{align*}
For the term \( A_l^{(2)} \), we have:
\begin{align*}
\|A_l^{(2)}\|_F &\leq \sum_{1\leq i<j \leq K_l} \|\Phi(\ldots,\Delta_{i,l},\ldots,\Delta_{j,l},\ldots)\|_F \\
&\leq \sum_{1\leq i<j \leq K_l} \left(\prod_{k\neq i,j}\|\mathcal G_{k,l}\|_F\right) \cdot \|\Delta_{i,l}\|_F \cdot \|\Delta_{j,l}\|_F\\
&= O(\rho^2).
\end{align*}
Similarly, higher-order terms are all $O(\rho^k)$ with $k\geq 2$, and we can bound:
\begin{align*}
    \|\tilde{\mathcal T}_l^{(t)} - \mathcal T_l^{(t)}\|_F
    &\leq \|A_l^{(1)}\|_F + \|A_l^{(2)}\|_F + \|A_l^{(>2)}\|_F\\
    &\leq O(\rho) + O(\rho^2) + O(\rho^3) + \cdots\\
    &= O(\rho).
\end{align*}
Thus, we have:
$$\|\tilde g_{i,l}^{(t)} - g_{i,l}^{(t)}\|_F  = O(\rho \hat L).$$
Combining the bounds, we have:
\begin{align*}
R^{(t)}_{ij,l} &\leq 2\rho \|\tilde g_{i,l}^{(t)} - g_{i,l}^{(t)}\|_F + 2\rho \|\tilde g_{j,l}^{(t)} - g_{j,l}^{(t)}\|_F \\
&= O(\rho^2 \hat L).
\end{align*}
The proof is complete.
\end{proof}

Then we can extend to a multi-layer version of Theorem~\ref{thm:sam-norm-deviation}. We define layer-wise Norm Deviation $Q_l:= \sum_{k=1}^{K_l} \left(\|\mathcal G_{k,l}\|_F^2 - \frac{1}{K_l}\sum_{i=1}^{K_l} \|\mathcal G_{i,l}\|^2_F\right)^2.$
\begin{theorem}[Norm Deviation Dynamics under SAM]
\label{thm:multi-layer-sam-norm-deviation}
Applying the update steps \eqref{eq:multi-layer-sam-update-step} with infinitesimal stepsize $\eta \to 0$,
the gradient flow of SAM satisfies:
\[
\frac{dQ_l}{dt} = 4\rho u_D^{(t)} K_l \cdot \mathrm{Cov}\left(\|\mathcal G_{k,l}^{(t)}\|_F^2, \|g_{k,l}^{(t)}\|_F^2\right) + O(\rho^2 \hat L).
\]
\end{theorem}
\begin{proof}
For simplicity, let $\bar{s} := \frac 1K \sum_{i=1}^{K_l} \|\mathcal G_{k,l}^{(t)}\|^2_F$ and 
$\bar{\gamma} := \frac {1}{K_l} \sum_{i=1}^{K_l} \|g_{k,l}^{(t)}\|^2_F$. The dynamics of $Q$ 
can be expressed as:
\begin{align*}
\frac{dQ_l}{dt} &= \frac{d}{dt}\sum_{k=1}^{K_l} \left(\|\mathcal G_{k,l}^{(t)}\|_F^2 - \bar{s}\right)^2 \\
&= \sum_{k=1}^{K_l} 2\left(\|\mathcal G_{k,l}^{(t)}\|_F^2 - \bar{s}\right) \frac{d}{dt}\left(\|\mathcal G_{k,l}^{(t)}\|_F^2 - \bar{s}\right)\\
&= 4\rho u_D^{(t)} \sum_{k=1}^{K_l} \left(\|\mathcal G_{k,l}^{(t)}\|_F^2 - \bar{s}\right)  \left(\|g_{k,l}^{(t)}\|_F^2 - \bar{\gamma}\right) + O(\rho^2 \hat L)\\
&= 4\rho u_D^{(t)} K_l\cdot \mathrm{Cov}\left(\|\mathcal G^{(t)}_{k,l}\|_F^2,\|g_{k,l}^{(t)}\|_F^2\right) + O(\rho^2 \hat L),
\end{align*}
where the third equality holds after summing the results of Theorem~\ref{thm:multi-layer-sam-dynamics}.
\end{proof}
\section{Experiment details}

All experiments are implemented with \texttt{torch}~\cite{paszke2019pytorch} library in python. The tensor completion experiments are done using CPU AMD EPYC 7413 (24C/48T, 2.65GHz, 128M cache). The experiments on tensorized ResNet-20/-32, FLoRA fine-tuning, and some of LoRETTA fine-tuning tasks are performed on either single NVIDIA RTX A5000 with 24 gigabytes (GB) of GPU memory (VRAM) or single NVIDIA RTX A6000 with 48 GB VRAM. ImageNet experiments are run on 7 NVIDIA RTX A5000's using multiple GPU recipes. Full fine-tuning experiments of LoRETTA are done on a single NVIDIA H100 Tensor Core GPU with 80GB VRAM when the VRAM requirements exceed 48 GB.

\subsection{Tucker Decomposition for Tensor Completion}
Details for  
the real-world data
COVID dataset can be found in the official user guide of tensorly~\cite{kossaifi2019tensorly}.

The data set is formatted as a three-mode tensor of samples, antigens, and receptors with shape $(438,6,11)$. We fix $\rho=0.01$ for SAM and $\alpha=0.001$ for DAS. 

\paragraph{Metric.} We use $R^2$ metrics. 

$$R^2 = 1 - \frac{SSE}{SST}$$
where:
\begin{itemize}
    \item     SSE is the Sum of Squared Errors (also called the residual sum of squares). It represents the unexplained variance. The formula is:
    $ SSE = \sum_{i=1}^{n} (y_i - \hat{y}_i)^2 $
    Here, $y_i$ is the actual observed value and $\hat{y}_i$ is the value predicted by the model.
    \item           SST is the Total Sum of Squares. It represents the total variance in the data. The formula is:
    $ SST = \sum_{i=1}^{n} (y_i - \bar{y})^2 $
    Here, $\bar{y}$ is the mean of the observed data.
\end{itemize}

\paragraph{Optimizers.} We use a base optimizer, ADAM. The hyperparameter for ADAM is the default setting of torch, \textit{i.e.}, \texttt{lr=0.001, betas=(0.9, 0.999), weight\_decay=0.01}. Each Tucker decomposition is optimized for $50000$ iterations.

\paragraph{Experiment results with standard deviation.} Due to space limitations, we did not show the standard deviation. Here we show the results with standard deviation in Table~\ref{tab:tucker-covid-apdx}.

\begin{table}[t]
\centering
\begin{tabular}{l|cc}
    \toprule
    Method & HOOI & ADAM  \\
    \midrule
    $\mathrm{R}^2$ score & $0.9268_{\pm 0.0063}$ & $0.9482_{\pm 0.0012}$ \\
    \midrule
    Method & SAM & DAS \\
    \midrule
    $\mathrm{R}^2$ score & $0.9485_{\pm 0.0013}$ & $0.9484_{\pm 0.0017}$ \\
    \bottomrule
\end{tabular}
\caption{Results of Tucker on COVID dataset.}
\label{tab:tucker-covid-apdx}
\end{table}
    
\subsection{Training Tensorized Neural Networks from Scratch}
A cosine annealing scheduler is applied with an initial learning rate of $0.1$ and batch size $128$. We use a weight decay of $0.0005$ and a momentum of $0.9$. Hyperparameter optimization is done using trials trained with $80$ epochs. We have used CP and Tensor-Ring (TR) in the tensorized ResNet-20. Convolution layer weights are inherently 4-order tensors, and we use 4-order CP decompositions to parameterize the layer. For Tensor-Ring, we follow existing works~\cite{wang2018wide,li2022heuristic,cao2024learning} and use the tensorization as in Table~\ref{tab:resnet shape}, to reshape the 4-order convolution layer weights to 5-order or 7-order tensors and use TR cores to parameterize layers. The uncompressed baseline for ResNet-32 on CIFAR-10 is 91.65.

We initialize the tensorized neural networks using the scheme proposed in~\cite{pan2022unified}, which ensures the expectation of signal flow consistent with the behavior of He initialization~\cite{he2016deep}. 

As discussed in the main paper, for both SAM and DAS, the corresponding hyperparameters $\rho$ and $\alpha$ are tuned independently over the shared search space $\{0.001, 0.005, 0.01, 0.05, 0.1\}$ using a validation split from the training set. We use a fixed validation split with 5000 samples from the train set of CIFAR-10. This results in the train/valid set having a size of 45000/5000. After the best hyperparameters (with the largest validation accuracy) are found, models are trained from scratch using the best choices using the whole 50000 training samples. Hyperparameter optimizations are implemented using Optuna~\cite{akiba2019optuna}.

\subsection{Tensorized Neural Networks under Label Noise}
We use a ResNet-32 with each convolution layer tensorized and parameterized with Tensor-Ring using the tensorization shape in Table~\ref{tab:resnet shape}. We use a uniform rank of 15. This results in a TR-ResNet-32 with 0.20M parameters, compared to the original model with 0.46M parameters. We initialize the tensorized neural network using the initialization in~\cite{pan2022unified}.
We use $\rho=0.05$ for SAM and $\alpha=0.1$ for DAS. We use a cosine annealing decreasing scheduler with an initial stepsize of 0.1. Momentum is set as 0.9, models are trained for 200 epochs, and weight decay is 0.0001. 
\begin{table}[]
    \centering
    \begin{tabular}{lccc}
    \toprule
    Layer & Tensorization & ResNet-32 & -20\\ \midrule
      \textsc{conv} & $(3, K^2,4, 2, 2)$  & & \\
      \textsc{conv} & $ (4, 2, 2, K^2,4, 2, 2)$ & $\times 10$ & $\times 5$ \\
      \textsc{conv} & $(4, 2, 2, K^2,4, 4, 2)$ & &  \\
      \textsc{conv} & $(4, 4, 2,K^2,4, 4, 2) $ &$\times 9$ & $\times 4$\\
      \textsc{conv} & $(4, 4, 2,K^2,4, 4, 4) $ & & \\
      \textsc{conv}  & $(4, 4, 4,K^2,4, 4, 4) $ &$\times 9 $& $\times 4$\\
      \textsc{fc}    & Not compressed& &\\
    \bottomrule
    \end{tabular}
    \caption{The tensorization and details of ResNet model in TR formats. The numbers of blocks in ResNet-32 and ResNet-20 are shown. $K$ is the kernel size and $K=3$ in ResNets. We keep the final output layer uncompressed.}
    \label{tab:resnet shape}
\end{table}

\subsection{Finetuning Pre-trained Models after Compression}

The pre-trained 
ResNet-18 is available in torchvision.
The TT decomposition ranks suggested in \cite{yin2021towards}
is available in the appendix of~\cite{yin2021towards}.
We use $\rho=0.05$ for SAM and $\alpha=0.01$ for DAS. We use a batch size of $2048$, a peak stepsize of $0.008$ with a one-cycle
scheduler~\cite{smith2019super}, a weight decay of $0.00001$, and a momentum of $0.9$. 
\begin{table}
\centering
\begin{tabular}{ll}
\toprule
\textbf{Layer}      & \textbf{TT-ranks}         \\ 
\midrule
\texttt{layer1.0.conv1} & \texttt{[1, 64, 64, 1]}   \\
\texttt{layer1.0.conv2} & \texttt{[1, 64, 64, 1]}   \\
\texttt{layer1.1.conv1} & \texttt{[1, 64, 64, 1]}   \\
\texttt{layer1.1.conv2} & \texttt{[1, 64, 64, 1]}   \\
\texttt{layer2.0.conv1} & \texttt{[1, 120, 60, 1]}  \\
\texttt{layer2.0.conv2} & \texttt{[1, 100, 100, 1]} \\
\texttt{layer2.1.conv1} & \texttt{[1, 100, 100, 1]} \\
\texttt{layer2.1.conv2} & \texttt{[1, 100, 100, 1]} \\
\texttt{layer3.0.conv1} & \texttt{[1, 200, 150, 1]} \\
\texttt{layer3.0.conv2} & \texttt{[1, 135, 135, 1]} \\
\texttt{layer3.1.conv1} & \texttt{[1, 135, 135, 1]} \\
\texttt{layer3.1.conv2} & \texttt{[1, 135, 135, 1]} \\
\texttt{layer4.0.conv1} & \texttt{[1, 320, 200, 1]} \\
\texttt{layer4.0.conv2} & \texttt{[1, 170, 170, 1]} \\
\texttt{layer4.1.conv1} & \texttt{[1, 170, 170, 1]} \\
\texttt{layer4.1.conv2} & \texttt{[1, 170, 170, 1]} \\
\bottomrule
\end{tabular}
\caption{TT-ranks for TT-ResNet-18.}
\label{tab:resnet18-tt-ranks}
\end{table}

\subsection{Finetuning with FLoRA}

FLoRA can be used for the adaptation of high-order tensors, such as convolution layer weights. Since our experiments require only tuning linear layers, we introduce only the 2-order case. For a pre-trained linear layer with shape $O\times I$, FLoRA~\cite{si2025maintaining} models the update $\Delta W = AGB^\intercal$ with $A\in\mathbb R^{O\times r}$, $G\in\mathbb R^{r\times r}$, and $B\in\mathbb R^{I\times r}$, resulting in a three-core tensorized model. FLoRA requires $r\times r$ extra parameters compared to LoRA with the sa.me rank, which is significantly small with small $r$. We experimented on RoBERTa-large~\cite{liu2019roberta}. In the LoRA and FLoRA experiments, the target modules of fine-tuning are query linear and value linear layers. 

Our codebase is heavily built on~\cite{malladi2023fine}. For the detailed scripts for building the datasets, please refer to the repository of~\cite{malladi2023fine}.
The hyperparameters for our experiments are summarized in Table~\ref{tab:flora-hp}. We report the test accuracy on the model trained on the best hyperparameter with the best validation performance for each few-shot datasets. Full fine-tuning baseline uses a batch size of $16$, and a stepsize from $\{0.00001, 0.00003, 0.00005\}$.

\begin{table}[]
    \centering
    \begin{tabular}{cc}
         \toprule
         Hyperparameter & Value \\
         \midrule
         Number of iterations & 1000 \\
         Batch size & 16 \\
         Stepsize & 0.0001, 0.0003, 0.0005\\
         LoRA/FLoRA rank & 8 \\
         LoRA/FLoRA $\alpha$ & 16 \\
         $\alpha$ for DAS & 0.1, 0.5, 0.7 \\
         $\rho$ for SAM & 0.001, 0.01, 0.1 \\
         \bottomrule
    \end{tabular}
    \caption{Hyperparameters for fine-tuning RoBERTa-large.}
    \label{tab:flora-hp}
\end{table}

\subsection{Finetuning with LoRETTA}

Specifically, LoRETTA first reshapes the weight matrix with shape $O\times I$ to $O_1\times\cdots\times O_n\times I_1\times\cdots \times I_m$, where $O = O_1\times\cdots\times O_n$ and $I = I_1\times\cdots \times I_m$, and then parameterizes the high-order tensor with TT cores.
The hidden size of OPT-6.7B is 4096, and we use a TT decomposition with shape $[4, 4, 16, 16, 16, 16, 4, 4]$ with rank $16$ to parameterize the update for the $4096\times4096$ query and value layer matrices. Our code base is built upon the repository of LoRETTA.

We fine-tune OPT-6.7B~\cite{zhang2022opt}, an autoregressive language model with 6.7B parameters on the SuperGLUE tasks (CB, BoolQ, WSC, COPA, ReCoRD)~\cite{wang2019superglue} and
generation tasks including SQuAD~\cite{rajpurkar2016squad} and DROP~\cite{dua2019drop}. Note that BoolQ and COPA use the accuracy metric, and the others use the F1 metric. We run all experiments for 3 epochs. Hyperparameters are summarized in Table~\ref{tab:loretta-hp}. Note that we use a batch size of 8 for COPA and WSC, 2 for BoolQ, and 1 for the other datasets. Except for the stepsize given in the Table~\ref{tab:loretta-hp}, we use ADAM as base optimizers with default settings, \textit{i.e.}, \texttt{betas=(0.9, 0.999), weight\_decay=0.01}.

\begin{table}[]
    \centering
    \begin{tabular}{cc}
         \toprule
         Hyperparameter & Value \\
         \midrule
         Stepsize & 0.0001, 0.0003, 0.0005\\
         LoRA/LoRETTA rank & 16 \\
         LoRA/FLoRA $\alpha$ & 16 \\
         $\alpha$ for DAS & 0.5, 1, 2 \\
         $\rho$ for SAM & 0.05, 0.1, 0.2 \\
         \bottomrule
    \end{tabular}
    \caption{Hyperparameters for fine-tuning OPT-6.7B.}
    \label{tab:loretta-hp}
\end{table}

\section{Limitations and Future Work}
A limitation of our theoretical analysis is the simplification of the optimization dynamics. The provided theorems, particularly Corollary~\ref{cor:sgd-norm-deviation}, are based on a standard gradient flow assumption for SGD, which predicts that the Norm Deviation is conserved over time.
However, our experiments—specifically those involving training from scratch and fine-tuning compressed models—utilize SGD with momentum for improved convergence. The introduction of a momentum term means the optimizer update is influenced by an accumulated velocity of past gradients, not just the gradient at the current step. Consequently, the Norm Deviation Q is not strictly preserved in our empirical settings, creating a small gap between the theoretical model and practical implementation.
A rigorous analysis of the Norm Deviation dynamics under momentum would require modeling the system with a second-order ordinary differential equation, often referred to as a ``heavy-ball'' analysis. Such an investigation exceeds the scope of this paper but represents a crucial direction for future work. This would provide a more complete understanding of how implicit regularization manifests in the presence of more complex, practical optimizers.

\end{document}